\documentclass{article}

\usepackage{microtype}
\usepackage{graphicx}
\usepackage{subfigure}
\usepackage{booktabs} 
\usepackage{caption}
\usepackage{hyperref}
\usepackage{soul}
\usepackage{threeparttable}

\usepackage[accepted]{icml2019}

\usepackage{multirow}
\usepackage{amsmath}
\usepackage{amssymb,amsthm}
\usepackage{appendix}
\newtheorem{theorem}{Theorem}
\newtheorem{lemma}{Lemma}

\icmltitlerunning{SWALP: Stochastic Weight Averaging in Low Precision Training}

\begin{document}

\twocolumn[
\icmltitle{SWALP: Stochastic Weight Averaging in Low-Precision Training}

\begin{icmlauthorlist}
\icmlauthor{Guandao Yang}{cornell}
\icmlauthor{Tianyi Zhang}{cornell}
\icmlauthor{Polina Kirichenko}{cornell}
\icmlauthor{Junwen Bai}{cornell} \\
\icmlauthor{Andrew Gordon Wilson}{cornell}
\icmlauthor{Christopher De Sa}{cornell}

\end{icmlauthorlist}

\icmlaffiliation{cornell}{Cornell University}

\icmlcorrespondingauthor{Guandao Yang}{gy46@cornell.edu}
\icmlcorrespondingauthor{Tianyi Zhang}{tz58@cornell.edu}

\icmlkeywords{Machine Learning, Low precision training, ICML}
\vskip 0.3in
]

\printAffiliationsAndNotice{}
\newcommand{\todo}[1]{\textcolor{red}{todo:#1}}
\newcommand{\polinacomment}[1]{\xxcomment{cyan}{P}{K}{comment:#1}}
\newcommand{\gy}[1]{\xxcomment{blue}{G}{Y}{#1}}
\newcommand{\xxcomment}[4]{\textcolor{#1}{[$^{\textsc{#2}}_{\textsc{#3}}$ #4]}}
\newcommand{\polina}[1]{\xxcomment{green}{P}{K}{#1}}

\begin{abstract}
Low precision operations can provide scalability, memory savings, portability, and energy efficiency. 
This paper proposes SWALP, an approach to low precision training that averages low-precision SGD iterates with a modified learning rate schedule. 
SWALP is easy to implement and can match the performance of \emph{full-precision} SGD even with all numbers quantized down to 8 bits, including the gradient accumulators.
Additionally, we show that SWALP converges arbitrarily close to the optimal solution for quadratic objectives, and to a noise ball asymptotically smaller than low precision SGD in strongly convex settings.
\end{abstract}

\section{Introduction}
Training deep neural networks (DNNs) usually requires a large amount of computational time and power.
As model sizes grow larger, training DNNs more efficiently and with less memory becomes increasingly important.
This is especially the case when training on a special-purpose hardware accelerator; such ML accelerators are in development and used in industry~\cite{jouppi2017datacenter,projectbrainwave}.
Many ML accelerators have limited on-chip memory, and many ML applications are memory bounded \cite{jouppi2017datacenter}.
It is desirable to fit numbers that are frequently used during the training process into the on-chip memory.
One of the useful techniques for doing this is \emph{low-precision computation}, since using fewer bits to represent numbers reduces both memory consumption and computational cost.

Training a DNN in low-precision usually results in worse performance compared to training in full precision.
Many techniques have been proposed to reduce this performance gap~\cite{dorefa-net,WAGE,scalable8bittraining,8bitfloat}.
One useful method is to compute forward and backward propagation with low-precision weights and accumulate gradient information in higher precision gradient accumulators~\cite{binaryconnect, WAGE, 8bitfloat}. 
Recently, \citet{8bitfloat} showed that one could eliminate the performance gap between low-precision and high-precision training by quantizing all numbers except the gradient accumulator to 8 bits without changing the network structure, establishing the state-of-the-art result in low-precision training.
Since gradient accumulators are frequently updated during training, it would be desirable to also represent and store them in low-precision (e.g. 8 bits).
In this paper, we will focus on the setting where all numbers including the gradient accumulators are represented in low precision during training.

Independently from low-precision computation, \emph{stochastic weight averaging} (SWA) \citep{SWA} has been recently proposed for improved generalization in deep learning. 
SWA takes an average of SGD iterates with a modified learning rate schedule and has been shown to lead to wider optima \citep{SWA}.
\citet{large-batch} also connect the width of the optimum and generalization performance.
A wider optimum is especially relevant in the context of low-precision training as it is more likely to contain high-accuracy low-precision solutions.
\citet{SWA} also observed that SWA works well with a relatively high learning rate and can tolerate additional noise during training.
Low-precision training, on the other hand, produces extra quantization noise during training and tends to underperform when the learning rate is low. Moreover, by averaging, one can combine weights that have been rounded up with those that been rounded down during quantization.
For these reasons, we hypothesize that SWA can boost the performance of low-precision training and that performance improvement is more significant than in the case of SWA applied to full-precision training.

\begin{figure}[t]
    \centering
    \includegraphics[width=\linewidth]{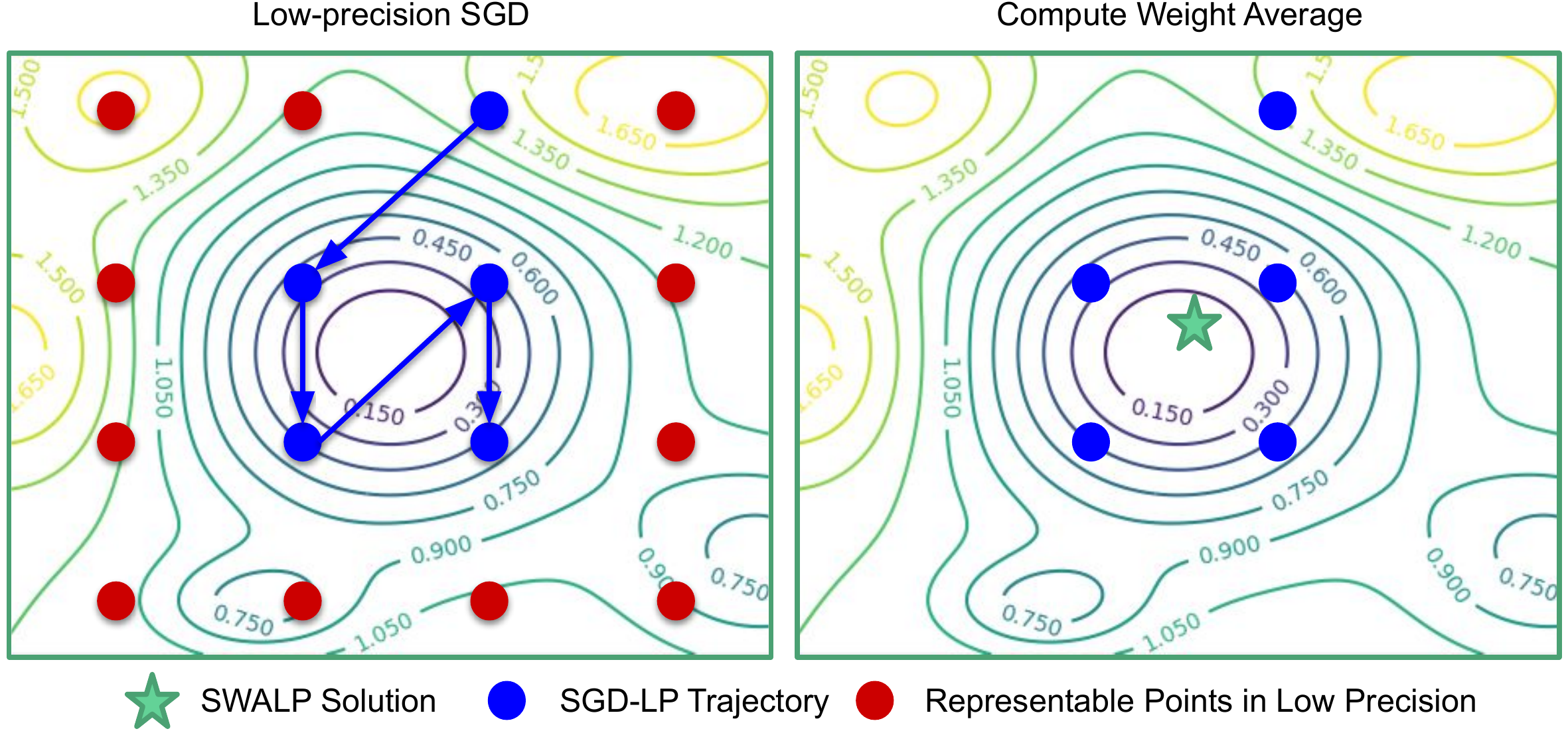}
    \caption{SWALP intuition. The trajectory of low-precision SGD, with a modified learning rate, over the training objective (with given contours), and the SWALP solution obtained by averaging. 
    }
    \label{fig:swalpexp}
\end{figure}

In this paper, we propose a principled approach using stochastic weight averaging while quantizing all numbers including the gradient accumulator and the velocity vectors during training. 
We prove that for quadratic objectives SWALP can converge to the optimal solution as well as to a smaller asymptotic noise ball than low-precision SGD for strongly convex objectives.
Figure~\ref{fig:swalpexp} illustrates the intuition behind SWALP. 
A quantized grid is only able to represent certain suboptimal solutions. By averaging we find centred solutions with better performance.
Empirically, we demonstrate that training with 8-bit SWALP can match the full precision SGD baseline in deep learning tasks such as training Preactivation ResNet-164~\cite{preact-resnet} on CIFAR-10 and CIFAR-100 datasets~\cite{CIFAR10}. 
In summary, our paper makes the following contributions:
\begin{itemize}
    \item We propose a principled approach to using stochastic weight averaging in low-precision training (SWALP) where all numbers including the gradient accumulators are quantized. Our method is simple to implement and has little computational overhead.
    \item We prove that SWALP can reach the optimal solution for quadratic objectives with no loss of accuracy from the quantization. For strongly convex objectives, we prove that SWALP converges to a noise ball that is asymptotically smaller than that of low-precision SGD.
    \item We show that our method can significantly reduce the performance gap between low-precision and full-precision training. Our experiment results show that 8-bit SWALP can match the full-precision SGD baseline on CIFAR-10 and CIFAR-100 with both VGG-16~\cite{VGG} and PreResNet-164. 
    \item We provide code at \url{https://github.com/stevenygd/SWALP}.
\end{itemize}
\section{Related Works}
\label{sec:related}

\paragraph{Low Precision Computation.} 
Many works in low precision computation focus on expediting inference and reducing model size. 
Some compress trained models into low precision~\cite{deepCompression}; others train models to produce low-precision weights for inference \cite{binary-net,ternary-quant,incremental-quant}. 
In contrast to works that focus on inference (test) time low-precision computation, our work focuses on low-precision training. 
Prior work on low precision training mostly explores two directions.
Some investigate different numerical representations and quantization methods~\cite{gupta2015deep,flexpoint,log-quant,dyanmic-fixed-point,8bitfloat}; others explore building specialized layers using low-precision arithmetic~\cite{WAGE,xor-net,dorefa-net,binaryconnect,scalable8bittraining}.
Our work is orthogonal to both directions since we improve on the learning algorithm itself.

\paragraph{Stochastic Weight Averaging.}
Inspired by the geometry of the loss function traversed by SGD with a modified learning rate schedule, \citet{SWA} proposed Stochastic Weight Averaging (SWA), which performs an equally weighted average of SGD iterates with cyclical or high constant learning rates. \citet{SWA} develop SWA for deep learning, showing improved generalization. While our work is inspired by \citet{SWA}, we focus on developing SWA for low-precision training.

\paragraph{Convergence Analysis.}
It is known that, due to quantization noise, low-precision SGD cannot necessarily produce solutions arbitrarily close to an optimum~\cite{training-quantized-network-deeper-understanding}.
A recently developed variant of low-precision SGD, HALP~\cite{HALP}, has the ability to produce such arbitrarily close solutions (for general convex objectives) by dynamically scaling its low-precision numbers and using variance reduction. 
We will show that SWALP can also achieve arbitrarily close-to-optimal solutions (albeit only for quadratic objectives), while being computationally simpler.
\citet{training-quantized-network-deeper-understanding} analyze low-precision SGD, and even provide a convergence bound for low precision SWA. However, they use SWA only as a theoretical condition, not as a suggested algorithm.
In contrast, we study SWA explicitly as a potential method that can improve low-precision training, and we use the averaging to improve the bound on the noise ball size.
QSGD~\cite{QSGD} studies the convergence properties of using low-precision numbers for communicating among parallel workers, and ZipML~\cite{ZipML} investigates the convergence properties of end-to-end quantization of the whole model. 
Our paper focuses on convergence properties of low-precision SWA, which we hope can be combined with these exciting prior works.

\section{Methods}\label{sec:method}

\citet{SWA} show that SWA leads to a wider solution, works well with a high learning rate, and is robust toward training noise.
These properties make SWA a good fit for low precision training, since a wider optimum is more likely to capture a representable solution in low precision, and low-precision training suffers from low learning rate and quantization noise.
We will first introduce quantization methods make training low-precision (Sec~\ref{sec:method-quant}), then present SWALP algorithm in Sec~\ref{sec:method-algo} and Sec~\ref{sec:method-quant-all}.

\subsection{Quantization}\label{sec:method-quant}

In order to use low-precision numbers during training, we define a \emph{quantization function} $Q$, which rounds a real number to be stored in fewer bits.
In this paper, we use \emph{fixed-point quantization} with \emph{stochastic rounding} to demonstrate the algorithm and analyze its convergence properties.
Recently, many sophisticated quantization methods have been proposed and studied \cite{log-quant,flexpoint,dyanmic-fixed-point,low-precision-multiply}.
We will use \emph{block floating point}~\cite{error-analysis} in our deep learning experiments (Sec~\ref{sec:expr}).

\textbf{Fixed Point Quantization.}
In stochastic rounding, numbers are rounded up or down at random such that $\mathbb{E}[Q(w)] = w$ for all $w$ that will not cause overflow.
Explicitly, suppose we allocate $W$ bits to represent the quantized number and allocate $F$ of the $W$ bits to represent the fractional part of the number.
The \emph{quantization gap} $\delta=2^{-F}$ represents the distance between successive representable fixed-point numbers.
The upper limit of the representable numbers is $u=2^{W-F-1}-2^{-F}$ and the lower limit is $l=-2^{W-F-1}$.
We write the quantization function as $Q_\delta: \mathbb{R} \rightarrow \mathbb{R}$ such that
\begin{align}
    Q_\delta(w) = \textstyle\begin{cases}
    \operatorname{clip}(\delta \left \lfloor \frac{w}{\delta} \right \rfloor, l,u) & \text{w.p. } \left \lceil \frac{w}{\delta} \right \rceil - \frac{w}{\delta} \\ 
    \operatorname{clip}(\delta \left \lceil \frac{w}{\delta} \right \rceil, l,u) & \text{w.p. } 1-(\left \lceil \frac{w}{\delta} \right \rceil - \frac{w}{\delta})
    \end{cases}\label{eq:fixed-point-SR}
\end{align}
where $\operatorname{clip}(x,l,u) = \max(\min(x,u), l)$.
This quantization method has been shown to be useful for theory~\cite{training-quantized-network-deeper-understanding} and has been observed to perform empirically better than quantization methods without stochastic rounding~\cite{gupta2015deep}.

\textbf{Block Floating Point (BFP) Quantization.}
Floating-point numbers have individual exponents, and fixed-point numbers all share the same fixed exponent.
For block floating-point numbers, all numbers within a block share the same exponent, which is allowed to vary like a floating-point exponent.
Suppose we allocate $W$ bits to represent each number in the block and $F$ bits to represent the shared exponent.
The shared exponent $E(\mathbf{w})$ for a block of numbers $\mathbf{w}$ is usually set to be the largest exponent in a block to avoid overflow \cite{error-analysis,Mixed-precision-training}. 
In our experiments, we simulated block floating point numbers by using the following formula to compute the shared exponent:
\begin{align*}
{\textstyle
E(\mathbf{w}) = \operatorname{clip}(\left \lfloor \log_2(\max_i|\mathbf{w}_i|)\right \rfloor, -2^{F-1}, 2^{F-1}-1)    
}
\end{align*}
We then apply equation~(\ref{eq:fixed-point-SR}) with $\delta$ replaced by $2^{-E(\mathbf{w}) + W - 2}$ to quantize all numbers in $\mathbf{w}$.

For deep learning experiments, BFP is preferred over fixed-point because BFP usually has less quantization error caused by overflow and underflow when quantizing DNN models \cite{error-analysis}.
We will discuss how to design appropriate blocks in Sec~\ref{sec:expr}, and show that choosing appropriate block design can result in better performance.

\subsection{Algorithm}\label{sec:method-algo}

\begin{algorithm}[t]
  \caption{SWALP}
  \label{alg:SWALP}
\begin{algorithmic}
\REQUIRE{
    Initial after-warm-up weight $w_0$; learning rate $\alpha$; total number of iterations $T$; cycle length $c$;
    random gradient samples $\nabla \tilde{f}(w_t)$;
    quantization function $Q$.
}
\STATE $\bar{w}_0 \leftarrow w_0$ \COMMENT{Accumulator for SWA (high precision)}
\STATE $m \leftarrow 1$ \COMMENT{Number of models that have been averaged}

\FOR{$t = 1, 2, \dots, T$ do} 
    \STATE {$w_t \leftarrow Q(w_{t-1} - \alpha \nabla \tilde{f}_t(w_{t-1}))$ \COMMENT{Training with weight quantization; $w_t$ is stored in low precision}}
    \IF{$t \equiv 0 \pmod{c}$}
        \STATE {$\bar{w}_m \leftarrow (\bar{w}_{m-1} \cdot m + w_t)/(m + 1)$  \COMMENT{Update model with weight averaging in high precision}}
        \STATE {$m \leftarrow m + 1$} \COMMENT{Increment model count}
    \ENDIF\ENDFOR
\STATE \textbf{return} $\bar{w}$
\end{algorithmic}
\end{algorithm}

In the \emph{warm-up phase}, we first run regular low-precision SGD to produce a pre-trained model $w_0$.
SWALP then continues to run low-precision SGD, while periodically averaging the current model weight into an accumulator $\bar{w}$, which SWALP eventually returns.
A detailed description is presented in Algorithm~\ref{alg:SWALP}.
SWALP is substantially similar to the algorithm in \citet{SWA}, except that we use a constant learning rate and low-precision SGD.
The convergence analysis in Sec~\ref{sec:convergence} will be all based on Algorithm~\ref{alg:SWALP}.

State-of-the-art low-precision training approaches usually separate weights from gradient accumulators~\cite{binaryconnect, dorefa-net, WAGE, 8bitfloat}.
Expensive computations in forward and backward propagation are done with the low-precision weights (e.g., 8 bits), while the gradient information is accumulated onto a higher precision copy of the weights (e.g. 16 bits).
Formally, the updating step with higher precision gradient accumulator can be written as $w_{t+1} = w_{t} - \alpha \nabla\tilde{f}_t(Q(w_t))$, where $w_{t}$ is the gradient accumulator and $Q(w_t)$ is the weight.
However, such an approach needs to store the high precision accumulators in low-latency memory (e.g. on-chip when running on an accelerator), which limits the memory efficiency.
SWALP quantizes the gradient accumulator so that we only need to store the low-precision model in low-latency memory during training. Meanwhile, the averaged model is accessed infrequently so that it can be stored in higher-latency memory (e.g. off-chip when running on an accelerator). 

Note that in Algorithm~\ref{alg:SWALP}, we only quantize the gradient accumulator while leaving the quantization of the gradient, the layer activations, and back-propagation signals in full precision.
In practice, however, it is desirable to quantize all above mentioned numbers.
We make this simplification for the convenience of the theoretical analysis in Sec~\ref{sec:convergence}, and following previous theoretical works in this space~\cite{training-quantized-network-deeper-understanding,HALP}.
We will discuss how to quantize other numbers during training in the next section.

\subsection{Quantizing Other Numbers}\label{sec:method-quant-all}

\begin{algorithm}[t]
  \caption{SWALP with all numbers quantized.}
  \label{alg:SWALP-all}
\begin{algorithmic}
\REQUIRE{
    $L$ layers DNN $\{f_1, \dots, f_L\}$;
    Scheduled learning rate $\alpha_t$; Momentum $\rho$;
    Initial weights $w_0^{(i)}, \forall l \in [1,L]$;
    Total iterations $T$;
    Warm-up iterations $S$; 
    Cycle length $c$;
    Quantization functions $Q_W$, $Q_A$, $Q_G$, $Q_E$, and $Q_M$; 
    Loss function $\mathcal{L}$;
    Data batch sequence $\{(x_i, y_i)\}_{i=1}^T$.
}
\STATE {$\bar{w}_0^{(l)} \leftarrow 0$, $\forall l \in [1,L]$; $m \leftarrow 0$}
\FOR{$t = 1, 2, \dots, T$} 
    \STATE{\textbf{1. Forward Propagation:}}
    \STATE{$a_t^{(0)} = x_i$}
    \STATE{$a_t^{(l)} = Q_A(f_l(a_t^{(l-1)}, w_t^{(l)}))$, $\forall l \in [1, L]$}

    \STATE{\textbf{2. Backward Propagation:}}
    \STATE{$e_t^{(L)} = \nabla_{a_t^{(L)}}\mathcal{L}(a_t^{(L)},y_t)$}
    \STATE{$e_t^{(l-1)} = Q_E(
    \frac{\partial f_l(a_t^{(l)})}{\partial a_t^{(l-1)}}
    e_t^{(l)}
    )$, $\forall l \in [1, L]$}
    \STATE{$g_t^{(l)} = Q_G(
        \frac{\partial f_l}{\partial w_t^{(l)}}
        e_t^{(l)}
    )$, $\forall l \in [1,L]$}
    
    \STATE{\textbf{3. Low Precision SGD Update (with momentum):}}
    \STATE {$v_t^{(l)} \leftarrow \rho Q_M(v_{t-1}^{(l)}) + g_t^{(l)}$, $\forall l \in [1,L]$}
    \STATE {$w_t^{(l)} \leftarrow Q_W(w_{t-1} - \alpha_t v_t^{(l)})$, $\forall l\in [1,L]$ }

    \STATE{\textbf{4. High Precision SWA Update:}}
    \IF{$t > S$ and $(t-S) \equiv 0 \pmod{c}$}
        \STATE {$\bar{w}_m^{(l)} \leftarrow (\bar{w}^{(l)}_{m-1} \cdot m + w^{(l)}_t)/(m + 1)$, $\forall l\in [1,L]$}
        \STATE {$m \leftarrow m + 1$}
    \ENDIF\ENDFOR
\STATE \textbf{return} $\bar{w}$
\end{algorithmic}
\end{algorithm}

In order to train DNNs with low-precision storage, we need to also quantize other numbers during training.
We follow prior convention to quantize the weights, activations, back-propagation errors, and gradients signals \cite{WAGE, 8bitfloat}.
Since we quantized the gradient accumulators $w_t$ into low-precision, there is no need to differentiate them from model weights.
To use momentum during training, we need to store the velocity tensors in low precision, so we modified the SGD update as follows:
\begin{align*}{\textstyle
    v_t &= \rho \cdot Q_M(v_{t-1}) + Q_G(\nabla \tilde{f}_t(w_{t-1})) \\
    w_t &= Q_W(w_{t-1} - \alpha \cdot v_t) 
}\end{align*}
where $Q_M$, $Q_G$, and $Q_W$ are quantizers for momentum, gradients, and weights respectively. 
For simplicity, we set $Q_M = Q_G$ (i.e. both quantized to 8 bits). 
We describe the details in Algorithm~\ref{alg:SWALP-all}.
Our deep learning experiments will use Algorithm~\ref{alg:SWALP-all} unless specified otherwise.

Although SWA adds minor computation and memory overhead by averaging weights, the fact that averaging could be done infrequently (i.e. once per epoch) and that the weight communication is one way (i.e. from accelerator to host) makes it possible to separate the low-precision training workload from the model averaging workload.
We could use hardware specialized for low-precision training to accelerate the SGD and to occasionally ship the weights in low precision to a separate device that computes weight averaging. 
For instance, one could train the model in low precision on a GPU, and the averaging could be computed on a CPU once per epoch. 
However, in this paper, we will focus on the statistical properties of SWALP and will leave the hardware discussion to future work.
Though the averaging workload (i.e. Step(4) in algorithm~\ref{alg:SWALP-all}) is typically done in higher precision, we empirically show in Sec~\ref{sec:expr-perf} that SWALP can achieve comparable performance when the averaging is performed with low-precision computation.

\section{Convergence Analysis}\label{sec:convergence} 

In this section, we analyze the convergence of SWALP theoretically and compare it to SGD-LP.
Specifically, we first prove that SWALP can pierce the quantization noise ball of SGD-LP and can converge to the optimal solution for quadratic objectives (Sec~\ref{sec:convergence-quadratic}).
Then, we generalize this theory to strongly convex objectives (Sec~\ref{sec:convergence-strongly-convex}) where we show that SWALP converges to a noise ball with better asymptotic dependency on the number of bits compared to SGD-LP.
These results are empirically verified in Sec~\ref{sec:convergence-verify}.

\subsection{Convergence of SWALP for Quadratic Objectives}\label{sec:convergence-quadratic}

It is known that conventional low-precision SGD cannot obtain arbitrarily accurate solutions to optimization problems since it can only represent so much.
If the optimal parameter is not one of the representable low-precision numbers, then the best SGD-LP can possibly do is to output the closest representable number -- and it is not even guaranteed to do this.
One recent algorithm, HALP, circumvents this problem by dynamically re-centering and re-scaling the representable numbers to produce arbitrarily accurate solutions with low precision iterates~\cite{HALP}.
In this subsection, we will demonstrate that SWALP can also achieve this property for quadratic objectives with a simple training procedure. 
A detailed proof is included in the appendix.

Consider the quadratic objective function
$f(w) = \frac{1}{2}(w - w^*)^T A (w - w^*)$
for some symmetric matrix $A \in \mathbb{R}^{d \times d}$ and optimal solution $w^* \in \mathbb{R}^d$.
Assume $A \succeq \mu I$ for some constant $\mu > 0$, the strong convexity parameter of this function.
Suppose that we want to minimize $f$ using SWALP with gradient samples $\nabla \tilde f(w)$ that satisfy $\mathbb{E}[\nabla \tilde f(w)] = \nabla f(w) = A (w - w^*)$.
Suppose that the variance of these samples always satisfies $\mathbb{E}[\| \nabla \tilde f(w) - \nabla f(w) \|_2^2] \le \sigma^2$ for some constant $\sigma > 0$; this is a standard assumption used in the analysis of SGD.
Then we can prove the following. 

\begin{theorem}\label{thm:quadratic}
Suppose we run SWALP under the above assumptions with cycle length $c$ and $0 < \alpha < \frac{1}{2}\|A\|_2$.
The expected squared distance to the optimum of SWALP's output $\bar{w}$ is bounded by
\begin{align*}
{\textstyle
\mathbb{E}\left[\|\bar{w} - w^*\|^2 \right]
\leq
\frac{\|w_0 - w^*\|^2}{\alpha^2\mu^2T^2} + 
\frac{c(\alpha^2\sigma^2 + \frac{\delta^2d}{4})}{\alpha^2\mu^2T}
}.\end{align*}
\end{theorem}

Theorem~\ref{thm:quadratic} shows that SWALP will converge to the optimal solution at a $O(1/T)$ rate. 
Since $\mathbb{E}[\|\bar{w}_T - w^*\|_2^2]$ converges to $0$ regardless of what $\delta$ is, we can get an arbitrarily precise solution no matter what numerical precision we use for quantization is, as long as we train for enough iterations.
This result is surprising since SWALP has the same $O(1/T)$ asymptotic convergence rate as full precision SGD, even though SWALP cannot even evaluate gradient samples at points that are arbitrarily close to the optimal solution.

\begin{figure*}[t]
    \centering
    \includegraphics[width=0.9\textwidth]{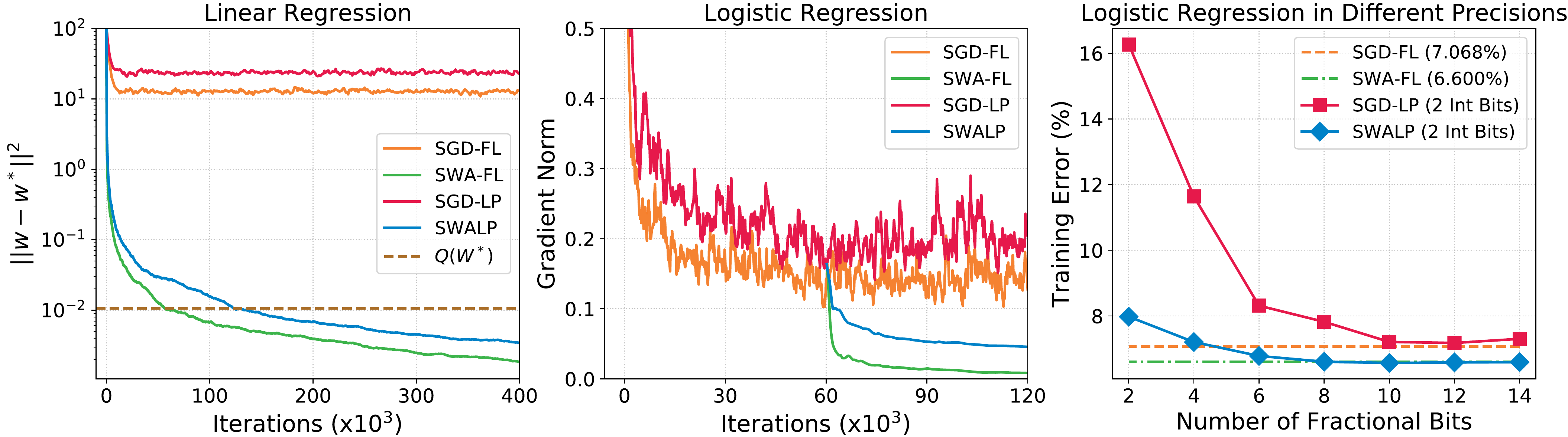}
    \caption{Empirical verification of two theorems with linear regression and logistic regression. (Left) SWALP converges below the quantization noise and to the optimal solution in linear regression; (Middle) SWALP can converge to a smaller noise ball than SGD-LP and SGD; (Right) SWALP requires less than half of the float bits to achieve the same performance compared to SGD-LP.
    }
    \label{fig:convergence}
\end{figure*}

\subsection{Convergence of SWALP for General Strongly Convex Objectives}\label{sec:convergence-strongly-convex}

To generalize Theorem~\ref{thm:quadratic} from quadratic settings to strongly convex settings, we want to construct a bound that is tight with our bound in the quadratic case: one that depends on how much the objective function differs from a quadratic.
A natural way to measure this is the Lipschitz constant $M$ of the second derivative of $f$, which is defined such that
\begin{align*}
{\textstyle
\forall x, y \in \mathbb{R}^d, \hspace{1em} \| \nabla^2 f(x) - \nabla^2 f(y) \|_2 \le M \| x - y \|_2
}
\end{align*}
where $\| \cdot \|_2$ denotes the matrix induced 2-norm, and $M = 0$ only if $f$ is a quadratic function.

We prove our result in two steps.
First, we bound the trajectory of low precision SGD within some distance of $w^*$ (a noise ball).
Then, we leverage a method similar to the proof of Theorem~\ref{thm:quadratic} to analyze the dynamics of SWALP, keeping track of the effect caused by the function not being quadratic as an extra error term that we bound in terms of $M$.
We give a tight bound that converges with an $O(1/T)$ rate to a noise ball with squared error proportional to $M^2$.

Let $f(w)$ be a function that is strongly convex with parameter $\mu$, Lipschitz continuous with parameter $L$, and has global minimum $w^*$.
Assume that we run SWALP with gradient samples $\nabla \tilde f_t$ that satisfy $\mathbb{E}[\nabla \tilde f(w)] = \nabla f(w)$.
Suppose the distance from the gradient samples to the actual gradient is bounded by some constant $G$ such that $\|\nabla\tilde{f}_t(w) - \nabla f(w)\|\leq G$ for all points $w$ that may appear in the course of execution.
Similar to Sec~\ref{sec:convergence-quadratic}, we assume that no overflow happens during quantization.

\begin{lemma}
\label{lemmaNoiseBall}
Under the above conditions, 
suppose that we run low-precision SGD with step size $\alpha = \sqrt{\frac{\delta^2 d}{G^2}}$. 
Assume $\delta$ is small enough so that it satisfies $(1-2\alpha\mu+\alpha^2L)^2 \leq 1-2\alpha\mu$ and $\alpha\mu < 1$.
If we run for $T$ iterations such that   
$ T \geq \frac{2G}{\mu\delta\sqrt{d}}\log{\left( \frac{\mu\|w_0 - w^*\|^2}{44G\delta\sqrt{d}} \right)}$, 
then for some fixed constant $\chi$ that is independent of dimension and problem parameters, 
\begin{align*}
{\textstyle
\mathbb{E}[\|w_T - w^*\|^4] \leq \frac{\chi^2 G^2\delta^2d}{\mu^2}.
}
\end{align*}
\end{lemma}

\begin{theorem}
\label{thm:SWALP}
Suppose that we run SWALP under the above conditions, with the parameters specified in Lemma~\ref{lemmaNoiseBall}.
Also, suppose that we first run a warm-up phase and start averaging at some point $w_0$ after enough iterations of low-precision SGD such that the bound of Lemma~\ref{lemmaNoiseBall} is already guaranteed to apply for this and all subsequent iterates.
Let $\bar{w}$ be the output of SWALP using cycle length $c$, and $\gamma = min(\alpha^2\mu^2c^2,1)$.
The expected squared distance to the optimum of the output of SWALP is bounded by
\begin{align*}
{\textstyle
    \mathbb{E}[\|\bar{w} - w^*\|^2] \leq
    \frac{3 \chi^2 M^2G^2\delta^2 d}{\mu^4} + 
    \frac{6 G^2 c}{\mu^2 T} +
    \frac{528 \sqrt{d}\delta G^3 c^2}{\gamma\mu T^2}.
}
\end{align*}
\end{theorem}

The first term $3\chi^2M^2G^2 \mu^{-4} \delta^2 d$ represents the squared errors caused by the noise ball, the asymptotic accuracy floor to which SWALP will converge given enough iterations.
The error caused by this noise ball is proportional to $M^2$, which measures how much the objective function differs from the quadratic setting. 
The second and third term converge to $0$ at a $O(1/T)$ rate, which shows that the whole bound will converge to the noise ball at a $O(1/T)$ rate.
Our proof leverages some techniques used in prior work \cite{francisbach}. 
In particular, \citet{francisbach} showed that one could provide a better bound in SGD using $M$, the third derivative of a strongly convex function.
The proofs of our results here are provided in detail in the appendix.

To the best of our knowledge, our result in Theorem~\ref{thm:SWALP} is the tightest known bound for low precision SWA. 
\citet{training-quantized-network-deeper-understanding} also provide results analyzing a  convergence bound for LP-SGD with weight averaging, but their bound is proportional to $\delta$, whereas ours is proportional to $\delta^2$.\footnote{Although their bound is stated in terms of the objective gap $f(\bar{w}_T) - f(w^*)$ whereas ours is the squared distance to the optimum, these metrics are directly comparable as they may differ by at most a factor of $\mu$: $2 (f(\bar{w}_T) - f(w^*)) \ge \mu \| \bar{w}_T - w^* \|_2^2$.}
If we consider a fixed-point representation using $F$ fractional bits, then our bound is proportional $\delta^2 = 2^{-2F}$ whereas the bound from prior work proportional to $\delta = 2^{-F}$.
Asymptotically, we can say that this halves the number of bits we need to decrease the noise ball by some factor, compared with the prior bound.
As this prior bound also describes the convergence behaviour, we can equivalently say that the number of bits in SWALP has double the effect on the noise ball, compared with SGD-LP.

To understand whether SWALP can achieve a better bound than SGD-LP, we need to answer the following question: can SGD-LP algorithm potentially achieve a better bound than the $O(\delta)$ bound proved in \citet{training-quantized-network-deeper-understanding}?
In the following theorem, we show that this is not possible:

\begin{theorem} \label{thm:lower-bound}
Consider the one-dimensional objective function $f(x) = \frac{1}{2}x^2$ with gradient samples $\tilde{f}'(w) = w + \sigma u$ where $u \sim \mathcal{N}(0,1)$. Compute $w_T$ recursively using the quantized SGD updating step: $w_{t+1} = Q_\delta(w_t - \alpha \tilde{f}'(w_t))$. Then there exists a constant $A > 0$ such that for all step size $\alpha > 0$, we have
$\lim_{T\to \infty} \mathbb{E}[w_T^2] \geq \sigma \delta A$.
\end{theorem}

The proof is provided in the appendix.
Theorem~\ref{thm:lower-bound} shows that there exists a strongly convex objective function such that $\mathbb{E}[(w_T - w^*)^2] \geq O(\delta)$.
This shows that the asymptotic lower bound for low-precision SGD with the gradient accumulator quantized at every step is $O(\delta)$, which is an asymptotically worse bound compared to the $O(\delta^2)$ bound obtained for SWALP.
Therefore, SWALP achieves a better asymptotic dependency on the quantization gap $\delta$.

\subsection{Experimental validation}\label{sec:convergence-verify}

\textbf{Linear regression.} 
We will use linear regression on a synthetic dataset to empirically verify Theorem~\ref{thm:quadratic}. 
For details on how we generate the synthetic dataset, please refer to appendix.
We train linear regression models using float SGD (SGD-FL), float SWA (SWA-FL), low precision SGD (SGD-LP), and SWALP.
Low-precision models use 8-bit fixed point numbers with 6 fractional bits (i.e., $\delta=2^{-6}$).

The results are displayed in Figure~\ref{fig:convergence} where we plot the square distance between $w_t$ (or $\bar w_t$ for SWALP) and the optimal solution $w^*$.
For reference, we also plot in Figure~\ref{fig:convergence} the squared distance between $Q(w^*)$ and $w^*$ to illustrate the size of quantization noise.
We observe that both SGD-LP and SGD-FL converge to a noise ball, and SGD-LP's noise ball is further away from $w^*$ than SGD-FL's noise ball---indicating that we are operating in a regime where quantization noise matters. 
SWA-FL and SWALP, on the other hand, both converge asymptotically towards the optimal solution.
Notably, SWALP pierces the quantization noise ball and even outperforms $Q(w^*)$.\footnote{Note that SWALP is able to do this because it represents the averaged model in full precision.}
The asymptotic convergence rate of SWALP and SWA-FL appears to be the same, and both appear to follow a $O(1/T)$ convergence trajectory, which validates our results in Theorem~\ref{thm:quadratic}.

\begin{table*}[th!]
\centering
\caption{ Test error (\%) on CIFAR-10 and CIFAR-100 for VGG16 and PreResNet-164 trained in different quantization setting.
}
\label{table:base-results}
{\sc\small
\begin{tabular}{@{}lccccccc@{}}
\toprule
& & \multicolumn{2}{c}{Float} & \multicolumn{2}{c}{8-bit Big-block} & \multicolumn{2}{c}{8-bit Small-block } \\ 
\cmidrule(r){3-4}\cmidrule(r){5-6}\cmidrule(r){7-8}
Dataset  & Model & SGD & SWA & SGDLP & SWALP & SGDLP & SWALP \\ \midrule
         & VGG16 & 6.81 $\pm{0.09}$  & 6.51 $\pm{0.14}$ & 8.23 $\pm{0.08}$ & 7.36 $\pm{0.05}$ & 7.61 $\pm{0.15}$ & 
         6.70 $\pm{0.12}$ \\
CIFAR-10 & PreResNet-164  & 4.63 $\pm{0.18}$  & 4.03 $\pm{0.10}$ & 6.51 $\pm{0.08}$ & 5.61 $\pm{0.17}$  & 5.83 $\pm{0.05}$ & 5.01 $\pm{0.14}$ \\
\midrule
          & VGG16          & 27.23 $\pm{0.17}$  & 25.93 $\pm{0.21}$ & 30.56 $\pm{0.67}$ & 28.66 $\pm{0.17}$ & 29.59 $\pm{0.32 }$ &
          26.65 $\pm{0.29}$ \\
CIFAR-100 & PreResNet-164  & 22.20 $\pm{0.57}$  & 19.95 $\pm{0.19}$ & 25.84 $\pm{0.52}$ & 24.92 $\pm{0.60}$ & 23.97 $\pm{0.08}$ & 21.76 $\pm{0.28}$ \\
\bottomrule
\end{tabular}}
\end{table*}

\textbf{Logistic regression.} 
To empirically validate Theorem~\ref{thm:SWALP}, we use logistic regression with L2 regularization on the MNIST dataset~\cite{MNIST}.
Following prior work~\cite{HALP, SVRG}, we choose $10^{-4}$ weight decay, which makes the objective a strongly convex function with $M\neq0$.
Similarly to our linear regression experiment, we use SGD-FL, SWA-FL, SGD-LP, and SWALP to train logistic regression models. 
For this experiment, we measure the norm of gradient at each iteration to illustrate the convergence of the algorithm; this is a more useful metric because MNIST is sparse and poorly conditioned, and it is a metric that has been used for logistic regression on MNIST in previous work~\cite{HALP}.
For SWALP and SGD-LP, we use 4-bit word length and 2-bit fractional length.
See appendix for detailed hyper-parameters.

In Figure~\ref{fig:convergence}, we again observe that SGD-LP converges to a larger noise ball than SGD-FL, which is caused by the additional quantization noise of low precision training.
Both SWA-FL and SWALP pierce the noise ball of SGD-FL.
However, unlike SWA-FL whose gradient norm appears to converge to zero, SWALP still hits a noise ball, albeit one that is much smaller than the one from SGD.
This validates Theorem~\ref{thm:SWALP}, which predicts that SWALP will converge to a noise ball when the problem setting is strongly convex yet non-quadratic (i.e. $M \neq 0$).

Figure~\ref{fig:convergence} also compares the training errors of logistic regression trained with different numbers of fractional bits, which determine $\delta$ in the theorem. 
Both SGD-LP and SWALP are trained with 2 integer bits and the same hyper-parameters, but we vary the number of fractional bits.
SWALP recovers the performance of the full precision SGD model with only 4 fractional bits, while SGD-LP needs 10 bits to do so.
This result validates the claim that SWALP needs only half the number of bits to achieve the same performance, which is predicted by Theorem~\ref{thm:SWALP} in terms of the asymptotic upper bound.
Although our theory bounds the convergence in terms of the training set, this conclusion still holds when evaluated on MNIST test set (see appendix).


\section{Experiments}\label{sec:expr}

In this section, we demonstrate the effectiveness of SWALP on non-convex problems in deep learning.

\textbf{Datasets.} 
We use the CIFAR~\cite{CIFAR10} and ImageNet~\cite{imagenet} datasets for our experiments.
Following prior work~\cite{SWA, WAGE}, we apply standard preprocessing and data augmentation for experiments on CIFAR datasets.
Preprocessing and data augmentation for ImageNet are adapted from the public PyTorch example~\cite{PyTorch}.

\textbf{Architectures.} 
We use the VGG-16~\cite{VGG} and Pre-activation ResNet-164~\cite{preact-resnet} on CIFAR datasets as in \citet{SWA, SWA-repo}.
For ImageNet experiments, we use ResNet-18~\cite{resnet}.

\textbf{Block Design.}
\citet{error-analysis} shows that appropriate block assignments are essential to achieve good performance with BFP.
In our experiment, we will test two block assignments: \emph{Big-block} and \emph{Small-block}.
The \emph{Big-block} design puts all numbers within the same tensor into the same block.
For example, the activation of a convolution layer may have shape $(B, C, W, H)$, and the Big-block design assigns one shared exponent for $B \times C \times W \times H$ numbers in this tensor.
The \emph{Small-block} design will follow \citet{error-analysis} and \citet{dorefa-net} except that we assign only one exponent for the following tensors: 1) bias in convolution and fully connected layers; 2) the learned scaling parameter in batch normalization layers, and 3) the learned shift parameter in batch normalization layers.
We empirically found that with these modifications, we can reduce memory consumption while regularizing the model. 
Storing our VGG16 network in 32-bit float requires 53.33 MB memory, while using 8-bit Small-block BFP with 8-bit shared exponents reduces this memory requirement to 14.59 MB.
Moreover, Small-block design uses only 5.2 KB more memory compared to the Big-block design, which is a negligible overhead.
In Sec~\ref{sec:expr-perf}, we will compare the performance between these two block assignment methods.

\subsection{CIFAR Experiments}\label{sec:expr-perf}
We demonstrate how SWALP (Algorithm~\ref{alg:SWALP-all}) is applied to train DNNs in image classification tasks on CIFAR-10 and CIFAR-100.
To examine how the block design of BFP affects performance, we train each network with both Big-block and Small-block BFP.
We use the reported hyper-parameters in \citet{SWA}, for full-precision SGD, SWA, and low-precision SGD runs.
SWALP's hyper-parameters are obtained from grid search on a validation set.
Please see the Appendix for more detail.

Table~\ref{table:base-results} shows results for different combinations of architecture, dataset, and quantization method.
First, the Small-block model outperforms all the Big-block models by a large margin.
SWALP also consistently outperforms SGD-LP across architectures and datasets, showing the promise of SWA for low precision training in deep learning.
Although the performance of SWALP does not match that of full precision SWA, the performance improvement of SWALP over SGD-LP is larger than that of SWA over SGD.
For example, for VGG16 trained with 8-bit Small-block BFP on CIFAR100, applying SWALP improves the SGD-LP performance by $2.94\%$ whereas the corresponding full-precision improvement is only $1.3\%$.
Notably, the performance of SWALP for VGG16 and PreResNet-164 trained with Small-block BFP can match that of full-precision SGD. 
On CIFAR-100 dataset, SWALP with Small-block BFP even outperforms the full-precision SGD baseline by $0.58\%$ with VGG-16 and by $0.44\%$ with PreResNet-164.

\begin{figure}[t]
\centering
\includegraphics[width=\linewidth]{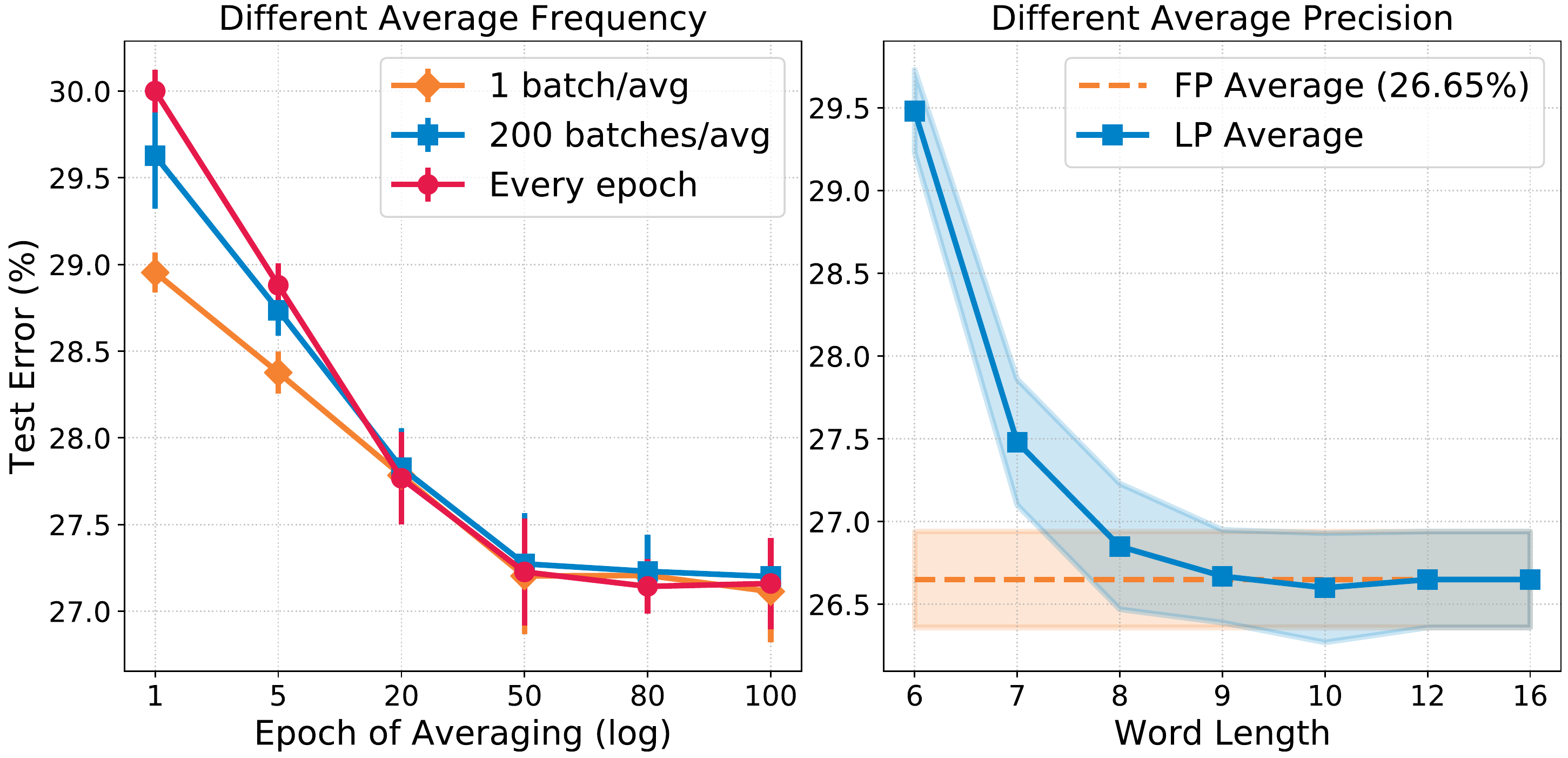}
\caption{
CIFAR-100 classification test error (\%). \textbf{Left:} Different averaging frequency. \textbf{Right:} Different averaging precision.
}
\label{fig:diffc_qswa}
\vspace{-2mm}
\end{figure}

\paragraph{Averaging in Different Frequency.}
Both Theorem~\ref{thm:quadratic} and Theorem~\ref{thm:SWALP} show that averaging more frequently leads to faster convergence, but changing $c$ does not affect the final convergence results. 
In this section, we empirically study the effect of $c$ using VGG-16 and CIFAR-100.
Previously, all runs compute the weight average once per epoch, following the convention from \citet{SWA}. 
We compare such default averaging frequency with higher frequencies including averaging every batch and every 200 batches. 
We keep the quantization method (Small-block BFP) and all other hyper-parameters unchanged.

The left panel of Figure~\ref{fig:diffc_qswa} show that averaging more frequently leads to faster convergence. 
For example, averaging every batch achieves an error rate of $28.85\%$ after one epoch, which is much lower than averaging only once per epoch (i.e. $30.00\%$).
However, the performance gap between high and low averaging frequency quickly disappears as we apply SWALP for more epochs. 
After 20 epochs, we observe almost no difference in performance between averaging frequencies. 
This result suggests that the averaging frequency does not affect final performance after sufficiently many epochs. 
We also observe that the test error keeps decreasing even after 20 epochs for all averaging frequencies, so averaging for more epochs is necessary for convergence.

\vspace{-1mm}
\paragraph{Averaging in Different Precision.}
We study the effect of averaging in different precisions while keeping quantization and other hyper-parameters unchanged.
The weight averaging is computed with low-precision operations as follows:
$\bar{w}_{m+1} \leftarrow Q_{\text{SWA}}((\bar{w}_{m} \cdot m + w_t)/(m + 1))$,
where $Q_{\text{SWA}}$ is a BFP quantizer with word length $W_{\text{SWA}}$.
In this experiment, we vary $W_{\text{SWA}}$ from $6$ to $16$ bits. 
The averaging updates are first computed in high precision before quantizing down to $W_{\text{SWA}}$ bits.
During inference, we will quantize the activation into $W_{\text{SWA}}$-bit Small-block BFP.

We report the results in the right panel of Figure~\ref{fig:diffc_qswa} of training a VGG-16 model with 8-bit Small-block BFP. 
We observe that the averaged weights can be computed in 9-bits with essentially no performance decrease compared to averaging in full precision.
When we use 8-bits BFP numbers to store the averaged model during training, there is a minor performance loss compared to full precision averaging.
That being said, the error rate ($26.85\%$) is still lower than those of the SGD-LP baseline ($29.59\%$) and the SGD-FP baseline($27.23\%$).
Averaging in lower than  8-bit precision tends to substantially hurt performance. 
This suggests that in order to fully realize the benefits of SWALP, one needs to compute the weight averaging in a slightly higher precision (i.e. for 8-bit weights, we need to compute the average in 9-bit).
These results suggest that we could replace step (4) of Algorithm~\ref{alg:SWALP-all} with this quantized averaging step to eliminate high-precision storage during training.
With such modifications, SWALP can produce a low-precision model that performs comparably to a \emph{full-precision} SGD model without any high precision storage.

\subsection{ImageNet Experiments}

We further evaluate SWALP's performance for a large scale image classification task, by training a ResNet-18 on ImageNet. We obtain the results for SGD and SGD-LP using hyper-parameters suggested by \citet{resnet}. 
For all low-precision experiments, we use Small-block BFP. Please see Appendix for the details on hyper-parameters. 
We present the results in Table~\ref{table:imagenet}. 

ImageNet contains substantially more information than CIFAR, and is more sensitive to hyper-parameter tuning; consequently, there is a greater drop in performance for ImageNet when using low precision computations: from $30.49\%$ with SGD to $36.56\%$ with SGD-LP. 
Although in preliminary experiments, SWALP does not entirely close this larger performance gap, achieving $34.89\%$ error after 10 epochs, it still leads to a substantial improvement in accuracy over SGD-LP. The performance gain with SWALP over SGD-LP is also greater than for SWA over SGD: after 10 epochs of weight averaging, there is a $1.67\%$ improvement in low precision compared to $0.84\%$ in full precision.

\begin{table}[t]
\centering
\caption{
ImageNet experiment results with ResNet-18. 
90+$X$ epochs of SWA (or SWALP) means running weight averaging for $X$ epochs starting at 90 epocth.
}\label{table:imagenet}
\begin{threeparttable}
\begin{tabular}{@{}ccc@{}}
\toprule
Model               & Epochs & Top-1 Error (\%) \\
\midrule
SGD                 & 90     & 30.49               \\
SWA                 & 90+10    & 29.74               \\ 
\midrule
SDGLP               & 90     & 36.56               \\ 
SWALP               & 90+10    & 34.89               \\ 
SWALP               & 90+30    & 34.34               \\ 
SWALP~\tnote{\dag}  & 90+30    & 34.18               \\ 
\bottomrule
\end{tabular}
\begin{tablenotes}
	\item[\dag] Averaging 50 times per epoch. 
\end{tablenotes}
\end{threeparttable}
\vspace{-2mm}
\end{table}

\vspace{-1mm}
\section{Conclusion}
\label{sec:concl}
We have proposed SWALP, a convenient approach to low-precision training that outperforms low-precision SGD and is competitive with full-precision SGD, even when trained in 8 bits. 
SWALP is based on averaging SGD iterates in low precision, motivated by the intuition that averaging could reduce the quantization noise introduced by stochastic rounding.
In the future, it would be exciting to explicitly consider loss geometry in building low precision solutions.

\clearpage
\section*{Acknowledgements}\label{sec:ack}
Polina Kirichenko and Andrew Gordon Wilson were supported by NSF IIS-1563887, an Amazon Research Award, and Facebook Research Award.
We thank Google Cloud Platform Research Credits program for providing computational resources.

\bibliography{ref}
\bibliographystyle{icml2019}

\clearpage
\onecolumn

\begin{appendices}

\newtheorem{innercustomgeneric}{\customgenericname}
\providecommand{\customgenericname}{}
\newcommand{\newcustomtheorem}[2]{%
  \newenvironment{#1}[1]
  {%
   \renewcommand\customgenericname{#2}%
   \renewcommand\theinnercustomgeneric{##1}%
   \innercustomgeneric
  }
  {\endinnercustomgeneric}
}

\newcustomtheorem{customthm}{Theorem}
\newcustomtheorem{customlemma}{Lemma}

\renewcommand{\qedsymbol}{$\blacksquare$}
\newcommand{\R}[0]{\mathbb{R}}
\newcommand{\Prob}[1]{\mathbf{P}\left( #1 \right) }
\newcommand{\Exv}[1]{\mathbb{E}\left[ #1 \right]}
\newcommand{\Exvud}[2]{\mathbb{E}_{#1}\left[ #2 \right]}
\newcommand{\norm}[1]{\left\| #1 \right\|}
\newcommand{\Abs}[1]{\left| #1 \right| }

\section{Overview}
The appendix will contain full proofs for all theorems and lemmas in the main paper, implementation details for all the experiments, statistics for the figures, and some additional results combining SWALP with other methods.
The proof for Theorem~\ref{thm:quadratic} will be presented in Sec~\ref{sec:appendix-thm-quadratic}. 
The proof for Lemma~\ref{lemmaNoiseBall} will be presented in Sec~\ref{sec:appendix-lemma}.
The proof for Theorem~\ref{thm:SWALP} will be shown in Sec~\ref{sec:thm4.3}.
The proof for Theorem~\ref{thm:lower-bound} will be shown in Sec~\ref{sec:thm-lowerbound}.
The implementation details of the linear regression and the logistic regression experiments will be in Sec~\ref{sec:linreg} and Sec~\ref{sec:logreg} respectively.
The implementation details for deep learning experiments will be presented in Sec~\ref{sec:dlexp}.
We show results combining SWALP with WAGE~\cite{WAGE} in Sec~\ref{sec:wage-swalp}.
In the rest sections, we will provide numbers used for the figure.

\section{Proof for Theorem~\ref{thm:quadratic}}\label{sec:appendix-thm-quadratic}

Here, we provide a more detailed proof of Theorem~\ref{thm:quadratic}.

Consider a quadratic objective function $f(w) = (w - w^*)^T A (w - w^*) / 2$ for some $A \in \mathbb{R}^{d \times d}$ and $w^* \in \mathbb{R}^d$ is the optimal solution.
Assume $A \succeq \mu I$, where $\mu > 0$ is the strong convexity parameter of this function.
Suppose that we want to minimize $f$ using SWALP with gradient samples $\nabla \tilde f(w)$ that satisfy $\mathbb{E}[\nabla \tilde f(w)] = \nabla f(w) = A (w - w^*)$.
Suppose that the variance of these samples always satisfies $\mathbb{E}[\| \nabla \tilde f(w) - \nabla f(w) \|_2^2] \le \sigma^2$ for some constant $\sigma$; this is a standard assumption used in the analysis of SGD.
Then we can prove the following:

\begin{customthm}{\ref{thm:quadratic}}
Suppose we run SWALP under the above assumptions with a cycle length $c$ and $0 < \alpha < \frac{1}{2}\|A\|_2$.
The expected squared distance to the optimum of SWALP's output is bounded by
\[{
\mathbb{E}\left[\|\bar{w}_T - w^*\|^2 \right]
\leq
\frac{\|w_0 - w^*\|^2}{\alpha^2\mu^2T^2} + 
\frac{c(\alpha^2\sigma^2 + \frac{\delta^2d}{4})}{\alpha^2\mu^2T}
}\]
\end{customthm}

\begin{proof}
We start by rewriting the iterates of low-precision SGD as
\begin{align*}
    w_{t+1} &= Q_\delta(w_t - \alpha \nabla \tilde{f}(w_t)) \\
    &= w_t - \alpha A(w_t-w^*) + \zeta_t
\end{align*}
where $\zeta_t$, the error term is explicitly defined as
\begin{align*}
    \zeta_t &= Q_\delta(w_t - \alpha \nabla \tilde{f}(w_t)) - w_t + \alpha A(w_t - w^*) \\
    &= Q_\delta(w_t - \alpha \nabla \tilde{f}(w_t)) - w_t + \alpha A(w_t - w^*) + \alpha \nabla \tilde{f}(w_t) - \alpha \nabla \tilde{f}(w_t) \\
    &= \alpha(A(w_t - w^*) - \nabla\tilde{f}(w_t)) + Q_\delta(w_t - \alpha \nabla \tilde{f}(w_t)) - (w_t - \alpha \nabla \tilde{f}(w_t)).
\end{align*}

We know that $\mathbb{E}[\zeta_i] = 0$ since
\begin{align*}
    \mathbb{E}[\zeta_t] &= \mathbb{E}[\alpha(A(w_t - w^*) - \nabla\tilde{f}(w_t)) + Q_\delta(w_t - \alpha \nabla \tilde{f}(w_t)) - (w_t - \alpha \nabla \tilde{f}(w_t))] \\
    &=\alpha \big(A(w_t - w^*) - \mathbb{E}[\nabla\tilde{f}(w_t)]\big) + \big(\mathbb{E}[Q_\delta(w_t - \alpha \nabla \tilde{f}(w_t))] - (w_t - \alpha \nabla \tilde{f}(w_t))\big) \\
    &= 0.
\end{align*}
Moreover, due to independence, we can bound the variance of $\zeta_t$ as
\begin{align*}
    \mathbb{E}\left[ \| \zeta_t \|^2 \right] 
    &= 
    \mathbb{E}\left[ \| \alpha(A(w_t - w^*) - \nabla\tilde{f}(w_t)) \|^2 \right] 
    + 
    \mathbb{E}\left[ \| Q_\delta(w_t - \alpha \nabla \tilde{f}(w_t)) - (w_t - \alpha \nabla \tilde{f}(w_t)) \|^2 \right] \\
    &\leq
    \alpha^2\sigma^2 + \frac{\delta^2 d}{4}
\end{align*}
where $d$ is the dimension of the solution and $\sigma$ is an upper bound for $A(w_t - w^*)$.
Then 
\begin{align*}
    w_{t+1} - w^* &= w_t - w^* - \alpha A(w-w^*) + \zeta_t \\
    &= (I-\alpha A)(w_t - w^*) + \zeta_t
\end{align*}
Expanding this formula we will have
\begin{align*}
    w_t - w^* = (I-\alpha A)^t(w_0-w^*) + \displaystyle\sum_{i=0}^{t-1} (I-\alpha A)^{t-i-1}\zeta_i
\end{align*}
Computing the distance from the average $\bar{w}_T$ to $w^*$, we get
\begin{align*}
    \bar{w}_K - w^* &= \frac{1}{K}\displaystyle\sum_{t=1}^K w_{ct} - w^* \\
    &=\frac{1}{K}\displaystyle\sum_{t=1}^K\bigg(
    (I-\alpha A)^{ct}(w_0-w^*) + \displaystyle\sum_{i=0}^{ct-1} (I-\alpha A)^{ct-i-1}\zeta_i
    \bigg) \\
    &=\frac{1}{K}\bigg(\displaystyle\sum_{t=1}^K(I-\alpha A)^{ct}\bigg)(w_0-w^*) 
    + \frac{1}{K}\displaystyle\sum_{t=1}^K\displaystyle\sum_{i=0}^{ct-1} (I-\alpha A)^{ct-i-1}\zeta_i. \\
\end{align*}
Note that the first term is a constant; let that be $\chi_K$. Now analyzing the variance:
\begin{align*}
\mathbb{E}\left[\|\bar{w}_K - w^*\|^2\right] &= \mathbb{E}\left[
    \left\|\chi_K + \frac{1}{K}\displaystyle\sum_{t=1}^K\displaystyle\sum_{i=0}^{ct-1} (I-\alpha A)^{ct-i-1}\zeta_i \right\|^2
\right] \\
&= \|\chi_K\|^2 + \mathbb{E}\left[
    \left\|\frac{1}{K}\displaystyle\sum_{t=1}^K\displaystyle\sum_{i=0}^{ct-1} (I-\alpha A)^{ct-i-1}\zeta_i \right\|^2
\right]
\\&\hspace{4em}+ 
    \frac{2}{K}\chi_K^T\displaystyle\sum_{t=1}^K\displaystyle\sum_{i=0}^{ct-1} (I-\alpha A)^{ct-i-1}\mathbb{E}[\zeta_i] \\
&= \|\chi_K\|^2 + \mathbb{E}\left[
    \left\|\frac{1}{K}\displaystyle\sum_{t=1}^K\displaystyle\sum_{i=0}^{ct-1} (I-\alpha A)^{ct-i-1}\zeta_i \right\|^2
\right] \\ 
&= \|\chi_K\|^2 + \frac{1}{K^2} \mathbb{E}\left[
    \left\|\sum_{i=0}^{cK-1}\displaystyle\sum_{t=\lfloor i/c \rfloor + 1}^{K} (I-\alpha A)^{ct-i-1}\zeta_i \right\|^2
\right] 
\end{align*}
Now leveraging the fact that the $\zeta_i$ are zero-mean and independent, we get
\begin{align*}
\mathbb{E}\left[\|\bar{w}_K - w^*\|^2\right]
&= \|\chi_K\|^2 + \frac{1}{K^2} \sum_{i=0}^{cK-1} \mathbb{E}\left[
    \left\| \sum_{t=\lfloor i/c \rfloor + 1}^{K} (I-\alpha A)^{ct-i-1}\zeta_i \right\|^2
\right] \\
&\leq \|\chi_K\|^2 + \frac{1}{K^2} \sum_{i=0}^{cK-1} \left\| \sum_{t=\lfloor i/c \rfloor + 1}^{K} (I-\alpha A)^{ct-i-1} \right\|_2^2 \mathbb{E}\left[
    \left\| \zeta_i \right\|^2
\right]  \\
&\leq \|\chi_K\|^2 + \frac{1}{K^2} \sum_{i=0}^{cK-1} \left\| \sum_{j=0}^{\infty} (I-\alpha A)^j \right\|_2^2 \mathbb{E}\left[
    \left\| \zeta_i \right\|^2
\right] 
\end{align*}
where in the final step the inequality holds because every term in the finite sum within the norm also must appear in the infinite sum.
Next, since $0 < \alpha < \|A\|_2/2$, we know the series sum $ \sum_{j=0}^{\infty} (I-\alpha A)^{j}$ will converge, and we have:
\begin{align*}
\left\| \sum_{j=0}^{\infty} (I-\alpha A)^{j} \right\|_2^2
= 
\left\|(I - (I-\alpha A))^{-1}\right\|_2^2 
=
\left\|\frac{1}{\alpha}A^{-1} \right\|_2^2
\leq
\frac{1}{\alpha^2\mu^2}.
\end{align*}

Putting this back we get,
\begin{align*}
\mathbb{E}[||\bar{w}_K - w^*||^2]
&\leq \|\chi_K\|^2 + 
    \frac{1}{K^2\alpha^2\mu^2} \displaystyle\sum_{i=0}^{cK-1}\mathbb{E}[\|\zeta_i\|^2]  \\
&\leq \|\chi_K\|^2 + 
    \frac{c}{K\alpha^2\mu^2} \left(\alpha^2\sigma^2 + \frac{\delta^2d}{4} \right).
\end{align*}

Finally, we analyze the term $\|\chi_K\|^2$, we will get
\begin{align*}
\|\chi_K\|^2 &= \left\|\frac{1}{K}\big(\displaystyle\sum_{t=1}^K(I-\alpha A)^{ct}\big)(w_0-w^*)\right\|^2 \\
&\leq \frac{1}{K^2}\left\|\displaystyle\sum_{t=1}^K(I-\alpha A)^{ct}\right\|_2^2 \  \left\|w_0 -w^*\right\|^2 \\
&\leq \frac{1}{K^2}\left\|\displaystyle\sum_{t=1}^\infty (I-\alpha A)^t\right\|_2^2 \ \left\|w_0 -w^*\right\|^2 \\
&= \frac{1}{K^2}\left\|(I - (I - \alpha A))^{-1}\right\|_2^2 \ \left\|w_0 -w^*\right\|^2 \\
&= \frac{1}{K^2}\left\|\frac{1}{\alpha}A^{-1}\right\|_2^2 \ \left\|w_0 -w^*\right\|^2 \\
&\leq \frac{1}{K^2\alpha^2\mu^2} \ \left\|w_0 -w^*\right\|^2
\end{align*}

Putting this bound together, we get:
\begin{align*}
    \mathbb{E}[||\bar{w}_K - w^*||^2]
&\leq \frac{1}{K^2\alpha^2\mu^2} \ \|w_0 -w^*\|^2 + 
    \frac{c}{K\alpha^2\mu^2} \left(\alpha^2\sigma^2 + \frac{\delta^2d}{4} \right)
\end{align*}
which is what we wanted to show.
\end{proof}

\section{Proof for Lemma~\ref{lemmaNoiseBall}}\label{sec:appendix-lemma}

Let's start by analyzing an update step
\[
  w_{t+1} = Q_{\delta}\left( w_t - \alpha \nabla \tilde f_t(w_t) \right)
\]
where $\Exv{\tilde f_t} = f$ and $f$ is strongly convex with parameter $\mu$, Lipschitz continuous with parameter $L$, and has a global minimum at $w^*$.
Further suppose that
\[
  \norm{ \nabla \tilde f_t(w) - \nabla f(w) } \le G
\]
for all $w \in \mathbb{R}^d$.

\begin{customlemma}{\ref{lemmaNoiseBall}}
  Under the above conditions, suppose that we run with step size
  \[
    \alpha
    =
    \sqrt{ \frac{\delta^2 d}{G^2} }
  \]
  and further suppose that $\delta$ is small enough that this step size satisfies
  $(1 - 2 \alpha \mu + \alpha^2 L)^2 \le 1 - 2 \alpha \mu$ and $\alpha \mu < 1$.
  Suppose that we run low-precision SGD for a number of time steps that is at least
  \[
    T
    \ge
    \frac{2 G}{\mu \delta \sqrt{d}} \log\left( \frac{\mu \norm{w_0 - w^*}^2}{44 G \delta \sqrt{d}} \right).
  \]
  Then
  \[
    \Exv{ \norm{w_T - w^*}^4 }
    \le
    \frac{44^2 G^2 \delta^2 d}{\mu^2}.
  \]
  \label{lemmaNewSingleIterBound}
\end{customlemma}

\begin{proof}
We start by bounding the quantity we are interested in at the next time step with
\begin{align*}
  \norm{w_{t+1} - w^*}^4
  &=
  \norm{Q_{\delta}\left( w_t - \alpha \nabla \tilde f_t(w_t) \right) - w^*}^4 \\
  &=
  \norm{w_t - w^* - \alpha \nabla f(w_t) - \left(
    w_t - \alpha \nabla \tilde f_t(w_t)
    -
    Q_{\delta}\left( w_t - \alpha \nabla \tilde f_t(w_t) \right)
    +
    \alpha \nabla \tilde f_t(w_t)
    -
    \alpha \nabla f(w_t)
  \right) }^4.
\end{align*}
If we define
\[
  u_t
  =
  -\left(
  w_t - \alpha \nabla \tilde f_t(w_t)
  -
  Q_{\delta}\left( w_t - \alpha \nabla \tilde f_t(w_t) \right)
  +
  \alpha \nabla \tilde f_t(w_t)
  -
  \alpha \nabla f(w_t)\right),
\]
then since we use unbiased rounding, $\Exv{u_t} = 0$.
And so,
\begin{align*}
  \Exv{ \norm{w_{t+1} - w^*}^4 }
  &=
  \Exv{ \norm{w_t - w^* - \alpha \nabla f(w_t) + u_t }^4 } \\
  &=
  \Exv{ \norm{w_t - w^* - \alpha \nabla f(w_t) }^4 }
  +
  \Exv{ 4 \norm{w_t - w^* - \alpha \nabla f(w_t) }^2
  u_t^T \left(w_t - w^* - \alpha \nabla f(w_t)\right) }
  \\&\hspace{2em}+
  \Exv{ 2 \norm{w_t - w^* - \alpha \nabla f(w_t) }^2 \norm{u_t}^2 }
  +
  \Exv{ 4 \left( u_t^T \left(w_t - w^* - \alpha \nabla f(w_t)\right) \right)^2 }
  \\&\hspace{2em}+
  \Exv{ 4 \norm{u_t}^2 u_t^T \left(w_t - w^* - \alpha \nabla f(w_t)\right) }
  +
  \Exv{ \norm{u_t}^4 } \\
  &=
  \Exv{ \norm{w_t - w^* - \alpha \nabla f(w_t) }^4 }
  +
  \Exv{ 2 \norm{w_t - w^* - \alpha \nabla f(w_t) }^2 \norm{u_t}^2 }
  \\&\hspace{2em}+
  \Exv{ 4 \left( u_t^T \left(w_t - w^* - \alpha \nabla f(w_t)\right) \right)^2 }
  +
  \Exv{ 4 \norm{u_t}^2 u_t^T \left(w_t - w^* - \alpha \nabla f(w_t)\right) }
  +
  \Exv{ \norm{u_t}^4 } \\
  &\le
  \Exv{ \norm{w_t - w^* - \alpha \nabla f(w_t) }^4 }
  +
  \Exv{ 6 \norm{w_t - w^* - \alpha \nabla f(w_t) }^2 \norm{u_t}^2 }
  \\&\hspace{2em}+
  \Exv{ 4 \norm{w_t - w^* - \alpha \nabla f(w_t) } \norm{u_t}^3 }
  +
  \Exv{ \norm{u_t}^4 }.
\end{align*}
From Holder's inequality, we know that
\[
  \Exv{X^m Y^n}
  \le
  \Exv{\Abs{X}^{m+n}}^{\frac{m}{m+n}} \Exv{\Abs{Y}^{m+n}}^{\frac{n}{m+n}},
\]
so
\begin{align*}
  \Exv{ \norm{w_{t+1} - w^*}^4 }
  &\le
  \Exv{ \norm{w_t - w^* - \alpha \nabla f(w_t) }^4 }
  +
  6 \Exv{ \norm{w_t - w^* - \alpha \nabla f(w_t) }^4 }^{\frac{1}{2}}
  \Exv{ \norm{u_t}^4 }^{\frac{1}{2}}
  \\&\hspace{2em}+
  4 \Exv{ \norm{w_t - w^* - \alpha \nabla f(w_t) }^4 }^{\frac{1}{4}}
  \Exv{ \norm{u_t}^4 }^{\frac{3}{4}}
  +
  \Exv{ \norm{u_t}^4 }.
\end{align*}
Now, it's a well known result for convex functions that
\begin{align*}
  \norm{w_t - w^* - \alpha \nabla f(w_t) }^2
  &=
  \norm{w_t - w^*}^2 
  - 
  2\alpha (w_t - w^*)^T \left( \nabla f(w_t) - \nabla f(w^*) \right)
  +
  \norm{ \nabla f(w_t) - \nabla f(w^*) }^2 \\
  &\le
  \norm{w_t - w^*}^2 
  - 
  2 \alpha \mu \norm{w_t - w^*}^2 
  +
  \alpha^2 L \norm{w_t - w^*}^2 \\
  &=
  (1 - 2 \alpha \mu + \alpha^2 L) \norm{w_t - w^*}^2,
\end{align*}
so
\[
  \norm{w_t - w^* - \alpha \nabla f(w_t) }^4
  \le
  (1 - 2 \alpha \mu + \alpha^2 L)^2 \norm{w_t - w^*}^4.
\]
On the other hand, we can bound the magnitude of $u$ with
\begin{align*}
  \norm{ u_t }
  &=
  \norm{ 
    w_t - \alpha \nabla \tilde f_t(w_t)
    -
    Q_{\delta}\left( w_t - \alpha \nabla \tilde f_t(w_t) \right)
    +
    \alpha \nabla \tilde f_t(w_t)
    -
    \alpha \nabla f(w_t)
    } \\
  &\le
  \norm{ 
    w_t - \alpha \nabla \tilde f_t(w_t)
    -
    Q_{\delta}\left( w_t - \alpha \nabla \tilde f_t(w_t) \right)
  }
  +
  \norm{
    \alpha \nabla \tilde f_t(w_t)
    -
    \alpha \nabla f(w_t)
  } \\
  &\le
  \delta \sqrt{d}
  +
  \alpha G.
\end{align*}
If we define $C = \delta \sqrt{d} + \alpha G$, then $\norm{ u_t } \le C$, and so
\begin{align*}
  \Exv{ \norm{w_{t+1} - w^*}^4 }
  &\le
  (1 - 2 \alpha \mu + \alpha^2 L)^2 \Exv{\norm{w_t - w^*}^4}
  +
  6 (1 - 2 \alpha \mu + \alpha^2 L) \Exv{\norm{w_t - w^*}^4}^{\frac{1}{2}}
  C^2
  \\&\hspace{2em}+
  4 \sqrt{1 - 2 \alpha \mu + \alpha^2 L} \cdot \Exv{\norm{w_t - w^*}^4}^{\frac{1}{4}}
  C^3
  +
  C^4 \\
  &\le
  (1 - 2 \alpha \mu + \alpha^2 L)^2 \Exv{\norm{w_t - w^*}^4}
  +
  6 C^2 \Exv{\norm{w_t - w^*}^4}^{\frac{1}{2}}
  +
  4 C^3 \Exv{\norm{w_t - w^*}^4}^{\frac{1}{4}}
  +
  C^4.
\end{align*}
Next, suppose we choose $\alpha$ small enough that $(1 - 2 \alpha \mu + \alpha^2 L)^2 \le 1 - 2 \alpha \mu$. Then,
\begin{align*}
  \Exv{ \norm{w_{t+1} - w^*}^4 }
  &\le
  (1 - 2 \alpha \mu) \Exv{\norm{w_t - w^*}^4}
  +
  6 C^2 \Exv{\norm{w_t - w^*}^4}^{\frac{1}{2}}
  +
  4 C^3 \Exv{\norm{w_t - w^*}^4}^{\frac{1}{4}}
  +
  C^4.
\end{align*}
Now, suppose that $\Exv{\norm{w_t - w^*}^4}$ is large enough that
\[
  \Exv{\norm{w_t - w^*}^4} \ge C^4.
\]
Then,
\begin{align*}
  \Exv{ \norm{w_{t+1} - w^*}^4 }
  &\le
  (1 - 2 \alpha \mu) \Exv{\norm{w_t - w^*}^4}
  +
  11 C^2 \Exv{\norm{w_t - w^*}^4}^{\frac{1}{2}}.
\end{align*}
And if we further suppose that $\Exv{\norm{w_t - w^*}^4}$ is large enough that
\[
  \alpha \mu \Exv{\norm{w_t - w^*}^4} \ge 11 C^2 \Exv{\norm{w_t - w^*}^4}^{\frac{1}{2}},
\]
which is equivalent to
\[
  \Exv{\norm{w_t - w^*}^4} \ge \left( \frac{11 C^2}{\alpha \mu} \right)^2,
\]
then
\begin{align*}
  \Exv{ \norm{w_{t+1} - w^*}^4 }
  &\le
  (1 - \alpha \mu) \Exv{\norm{w_t - w^*}^4}.
\end{align*}
So we've shown that $\Exv{ \norm{w_{t+1} - w^*}^4 }$ decreases exponentially until it no longer satisfies one of our two suppositions, which,  
as long as $\alpha \mu < 11$, will be
\[
  \Exv{\norm{w_t - w^*}^4} \ge \left( \frac{11 C^2}{\alpha \mu} \right)^2.
\]
In other words, independently of any suppositions on the magnitude of $\Exv{ \norm{w_{t+1} - w^*}^4 }$ (but still assuming $\alpha \mu < 11$), it will hold that
\begin{align*}
  \Exv{ \norm{w_{t+1} - w^*}^4 }
  &\le
  (1 - \alpha \mu) \max\left( \Exv{\norm{w_t - w^*}^4}, \left( \frac{11 C^2}{\alpha \mu} \right)^2 \right).
\end{align*}
And applying this recursively,
\begin{align*}
  \Exv{ \norm{w_T - w^*}^4 }
  &\le
  \max\left( (1 - \alpha \mu)^T \norm{w_0 - w^*}^4, \left( \frac{11 C^2}{\alpha \mu} \right)^2 \right) \\
  &\le
  \max\left( \exp(- \alpha \mu T) \norm{w_0 - w^*}^4, \left( \frac{11 C^2}{\alpha \mu} \right)^2 \right).
\end{align*}
It follows that, as long as $T$ is large enough that
\[
  \exp(- \alpha \mu T) \norm{w_0 - w^*}^4 \le \left( \frac{11 C^2}{\alpha \mu} \right)^2,
\]
which will occur when
\[
  T \ge \frac{2}{\alpha \mu} \log\left( \frac{\alpha \mu \norm{w_0 - w^*}^2}{11 C^2} \right),
\]
we will get 
\begin{align*}
  \Exv{ \norm{w_T - w^*}^4 }
  &\le
  \left( \frac{11 C^2}{\alpha \mu} \right)^2.
\end{align*}
Next, we simplify this fraction.
\begin{align*}
  \frac{C^2}{\alpha \mu}
  &=
  \frac{\left( \delta \sqrt{d} + \alpha G \right)^2}{\alpha \mu}.
\end{align*}
If we set $\alpha = G^{-1} \delta \sqrt{d}$ then
\begin{align*}
  \frac{C^2}{\alpha \mu}
  &=
  \frac{\left( \delta \sqrt{d} + G^{-1} \delta \sqrt{d} \cdot G \right)^2}{G^{-1} \delta \sqrt{d} \mu} \\
  &=
  \frac{\left( 2 \delta \sqrt{d} \right)^2 \cdot G}{\delta \sqrt{d} \mu} \\
  &=
  \frac{4 G \delta \sqrt{d}}{\mu}.
\end{align*}
It follows that our bound on $T$ becomes
\[
  T \ge \frac{2 G}{\mu \delta \sqrt{d}} \log\left( \frac{\mu \norm{w_0 - w^*}^2}{44 G \delta \sqrt{d}} \right),
\]
and our bound on $\Exv{ \norm{w_T - w^*}^4 }$ becomes
\begin{align*}
  \Exv{ \norm{w_T - w^*}^4 }
  &\le
  \left( \frac{44 G \delta \sqrt{d}}{\mu} \right)^2 \\
  &=
  \frac{44^2 G^2 \delta^2 d}{\mu^2}.
\end{align*}
This is what we wanted to prove.
\end{proof}

\newpage

\section{Proof for Theorem~\ref{thm:SWALP}}\label{sec:thm4.3}

\begin{customthm}{\ref{thm:SWALP}}
Suppose that we run SWALP under the above conditions, with the parameters specified in Lemma~\ref{lemmaNewSingleIterBound}.
Also, suppose that we first run a warm-up phase and start averaging at some point $w_0$ after enough iterations of low-precision SGD such that the bound of Lemma~\ref{lemmaNewSingleIterBound} is already guaranteed to apply for this and all subsequent iterates.
Let $\bar{w}$ be the output of SWALP using cycle length $c$, and $\gamma = min(\alpha^2\mu^2c^2,1)$.
The expected squared distance to the optimum of the output of SWALP is bounded by
\begin{align*}
    \mathbb{E}[\|\bar{w} - w^*\|^2] \leq
    \frac{5808 M^2G^2\delta^2 d}{\mu^4} + 
    \frac{6 G^2 c}{\mu^2 T} +
    \frac{528 \sqrt{d}\delta G^3 c^2}{\gamma\mu T^2}.
\end{align*}
\end{customthm}

\begin{proof}
Consider the case that $\bar{w}$ is output after $K$ averagings. Therefore, we could assume $cK \leq T < cK+c$, and $\bar{w}=\frac{1}{K}\sum_{t=0}^{K-1} w_{tc}$.
From rearranging the SGD update step, we have
\begin{align*}
  w_t - w_{t+1}
  &=
  Q_{\delta}\left(\alpha \nabla \tilde f_t(w_t) \right) \\
  &=
  \alpha H (w_t - w^*)
  +
  \alpha \left(
    \nabla f(w_t)
    -
    H (w_t - w^*)
  \right)
  +
  Q_{\delta}\left(\alpha \nabla \tilde f_t(w_t) \right)
  -
  \alpha \nabla f(w_t)
\end{align*}
where we let $H$ denote $H = \nabla^2 f(w^*)$.
For simplicity of the presentation, we let
\begin{align*}
    \phi_t &= \nabla f(w_t) - H ( w_t - w^*) \\
    \psi_t &= Q_\delta(\alpha \nabla \tilde{f}_t(w_t)) - \alpha \nabla f(w_t) \\
    \eta_t &= -\alpha \left(\nabla f(w_t) - H (w_t - w^*) \right)
              - Q_{\delta}\left(\alpha \nabla \tilde f_t(w_t) \right)
              + \alpha \nabla f(w_t) \\
           &= \alpha \phi_t + \psi_t
\end{align*}
Then  we can re-arrange the terms as following:
\begin{align*}
    w_{t+1} &= w_t - \alpha H(w_t - w^*) - \eta_t \\
    w_{t+1} - w^* &= (w_t - w^*) - \alpha H ( w_t - w^* ) - \eta_t \\
                  &= (I - \alpha H) (w_t - w^*) - \eta_t \\
    w_{t+c} - w^* &= (I - \alpha H)^c ( w_t - w^*) - \sum_{i=0}^{c-1} \eta_{t+i}(I-\alpha H)^{c-u-1} \\
    w_t - w_{t+c} &= (I - (I - \alpha H)^c) ( w_t - w^*) + \sum_{i=0}^{c-1}(I - \alpha H)^{c-i-1}\eta_{t+i} \\
    w_0 - w_{cK} &= \sum_{t=0}^{K-1}(I - (I-\alpha H)^c)(w_{tc} - w^*) + \sum_{t=0}^{K-1} \sum_{i=0}^{c-1} (I - \alpha H)^{c-i-1} \eta_{ct+i} \\
    \frac{1}{K}(w_0 - w_{cK}) &= 
        \frac{1}{K}\sum_{t=0}^{K-1}(I - (I-\alpha H)^c)(w_{tc} - w^*) + 
        \frac{1}{K}\sum_{t=0}^{K-1} \sum_{i=0}^{c-1} (I - \alpha H)^{c-i-1} \eta_{ct+i} \\
        &= 
        (I - (I-\alpha H)^c) \frac{1}{K}\sum_{t=0}^{K-1}(w_{tc} - w^*) + 
        \frac{1}{K}\sum_{t=0}^{K-1} \sum_{i=0}^{c-1} (I - \alpha H)^{c-i-1} \eta_{ct+i} \\
\end{align*}
We know $\bar{w} = \frac{1}{K}\sum_{t=0}^{K-1} w_{tc}$ and $cK \leq T < cK+c$. So $\frac{1}{K}\sum_{t=0}^{K-1}(w_{tc} - w^*) = \bar{w} - w^*$. 
\begin{align*}
\bar{w} - w^* &= (I - (I - \alpha H)^c)^{-1} \left[ 
    \frac{1}{K}(w_0 - w_{cK}) + \frac{1}{K}\sum_{t=0}^{K-1}\sum_{i=0}^{c-1} (I-\alpha H)^{c-i-1} \eta_{ct+i}
\right]
\end{align*}

Bounding the expected norm of $\bar{w} - w^*$, we get
\begin{align*}
    \Exv{\|\bar{w} - w^*\|} & \leq \frac{1}{K}\Exv{\left\|(I-(I - \alpha H)^c)^{-1} (w_0 - w_{cK})\right\|} \\
                            & \quad + \frac{\alpha}{K} \Exv{\left\|\sum_{i=0}^{c-1} (I - (I-\alpha H)^c)^{-1} (I - \alpha H)^{c-i-1} \phi_{ct+i}\right\|} \\ 
                            & \quad + \Exv{\left\|\frac{1}{K}\sum_{t=0}^{K-1}\sum_{i=0}^{c-1}(I - (I - \alpha H)^c)^{-1} (I - \alpha H )^{c-i-1} \psi_{ct+i}\right\|}
\end{align*}
Applying the inequality $(x+y+z)^2 \leq 3x^2+3y^2+3z^2$, we get
\begin{align*}
    \Exv{\|\bar{w} - w^*\|^2} & \leq \frac{3}{K^2}\Exv{\left\|(I-(I - \alpha H)^c)^{-1} (w_0 - w_{cK})\right\|^2} \\
                            & \quad + 3 \alpha^2\left(\frac{1}{K} \Exv{\left\|\sum_{i=0}^{c-1} (I - (I-\alpha H)^c)^{-1} (I - \alpha H)^{c-i-1} \phi_{ct+i}\right\|}\right)^2 \\ 
                            & \quad + 3\Exv{\left\|\frac{1}{K}\sum_{t=0}^{K-1}\sum_{i=0}^{c-1}(I - (I - \alpha H)^c)^{-1} (I - \alpha H )^{c-i-1} \psi_{ct+i}\right\|^2}
\end{align*}
We will now bound each of these three terms separately, starting with the first term:
\begin{align*}
    \frac{3}{K^2}\Exv{\left\| (I - (I - \alpha H)^c)^{-1} ( w_0 - w_{cK}) \right\|^2} 
    & \leq \frac{3}{K^2}\left\| (I - (I - \alpha H)^c)^{-1}\right\|_2^2 \Exv{\|w_0 - w_{cK}\|^2}
\end{align*}
By the result of Lemma~\ref{lemmaNoiseBall}, and Jensen's inequality, we will have
\begin{align*}
    \Exv{\|w_t - w^*\|^2} \leq \Exv{\|w_t - w^*\|^4}^{\frac{1}{2}} \leq \frac{44G\delta \sqrt{d}}{\mu}
\end{align*}
for any fixed $t$ during the executation of the algorithm (after the warm-up period). It follows that
\begin{align*}
    \Exv{\|w_0 - w_K\|^2}] 
    &\leq 2\Exv{\|w_K - w^*\|^2} + 2 \Exv{\|w_0 - w^*\|^2} \\
    &\leq \frac{176 G \delta \sqrt{d}}{\mu \alpha^2 K^2} = \frac{176 G^3}{\mu \delta \sqrt{d} K^2}
\end{align*}
where the last line followed by substituting $\alpha = \sqrt{\frac{\delta^2 d}{G^2}}$ from Lemma~\ref{lemmaNoiseBall}.
Since $H \succeq \mu I$, so $\|H\|_2 \geq \mu$. We have
\begin{align*}
    \|(I - (I - \alpha H)^c)^{-1}\|_2^2 \leq  (1 - (1 - \alpha \mu)^c)^{-2}
\end{align*}
Therefore, we get the following bound on the first term:
\begin{align*}
    \frac{3}{K^2}\Exv{\left\| (I - (I - \alpha H)^c)^{-1} ( w_0 - w_{cK}) \right\|^2} 
    & \leq \frac{1}{K^2}(1-(1-\alpha \mu)^c)^{-2}\frac{176 G\delta \sqrt{d}}{\mu}
\end{align*}
Note that if c is really small (i.e. $\alpha \mu c << 1$), then
\begin{align*} 
    (1- (1- \alpha \mu)^c)^{-2} \leq \frac{1}{\alpha^2 \mu^2 c^2}
\end{align*}
If $c$ is large, then $(1-(1 - \alpha \mu)^c)^{-2} \rightarrow 1$, so we can bound it by 
\begin{align*}
    (1- (1- \alpha \mu)^c)^{-2} \leq \frac{1}{min(\alpha^2\mu^2c^2, 1)} = \frac{1}{\gamma}
\end{align*}
Therefore, the first term is then bounded by:
\begin{align*}
    \frac{3}{K^2}\Exv{\left\| (I - (I - \alpha H)^c)^{-1} ( w_0 - w_{cK}) \right\|^2} 
    & \leq \frac{528 G \delta \sqrt{d}}{\gamma \mu K^2}
\end{align*}

Now we proceed to bound the second term:
\begin{align*}
    & 3 \alpha^2\left(\frac{1}{K} \Exv{\left\|\sum_{i=0}^{c-1} (I - (I-\alpha H)^c)^{-1} (I - \alpha H)^{c-i-1} \phi_{ct+i}\right\|}\right)^2 \\
    &\quad \leq \frac{3\alpha^2}{K} \left( \sum_{t=0}^{K-1} \Exv{\left\|\sum_{i=0}^{c-1} (I - \alpha H)^{c-i-1} \phi_{ct+i}\right\|^2} \right) \left\| (I - (I - \alpha H)^c)^{-1} \right\|_2^2
\end{align*}
Analyzing the the expectation term, 
\begin{align*}
    & \Exv{\left\| \sum_{i=0}^{c-1} (I - \alpha H)^{c-i-1}\phi_{ct+i} \right\|^2} \\
    & \quad = \left( \sum_{i=0}^{c-1} (1 - \alpha \mu )^{c-i-1} \right)^2 \Exv{\left\| \frac{1}{\sum_{i=0}^{c-1}(1-\alpha \mu)^{c-i-1}}\sum_{i=0}^{c-1} \frac{(1-\alpha\mu)^{c-i-1}}{(1-\alpha\mu)^{c-i-1}}(I - \alpha H)^{c-i-1}\phi_{ct+i} \right\|^2} \\
    & \quad = \left( \sum_{i=0}^{c-1} (1 - \alpha \mu )^{c-i-1} \right)^2 
    \Exv{\left\| 
        \sum_{i=0}^{c-1} 
            \frac{(1-\alpha\mu)^{c-i-1}}{\sum_{i=0}^{c-1}(1-\alpha \mu)^{c-i-1}}
            \frac{(I - \alpha H)^{c-i-1}}{(1-\alpha\mu)^{c-i-1}}\phi_{ct+i} \right\|^2} \\
     & \quad \leq \left( \sum_{i=0}^{c-1} (1 - \alpha \mu )^{c-i-1} \right)^2 
        \sum_{i=0}^{c-1} \frac{(1-\alpha\mu)^{c-i-1}}{\sum_{i=0}^{c-1}(1-\alpha \mu)^{c-i-1}}
        \Exv{\left\| 
            \frac{(I - \alpha H)^{c-i-1}}{(1-\alpha\mu)^{c-i-1}}\phi_{ct+i} \right\|^2}
        \quad (\text{Jensen's inequality}) \\
    & \quad = \left( \sum_{i=0}^{c-1} (1 - \alpha \mu )^{c-i-1} \right)\sum_{i=0}^{c-1} (1 - \alpha \mu)^{c-i-1} \Exv{\left\|\frac{(I-\alpha H)^{c-i-1}}{(1-\alpha\mu)^{c-i-1}}\phi_{ct+i}\right\|^2} \\
    & \quad \leq \left( \sum_{i=0}^{c-1} (1 - \alpha \mu )^{c-i-1} \right)\sum_{i=0}^{c-1} (1 - \alpha \mu)^{c-i-1} \Exv{\left\|\phi_{ct+i}\right\|^2} \quad 
    \left(H \succeq \mu I \implies \left\| \frac{(I - \alpha H)^{c-i-1}}{(1 - \alpha \mu)^{c-i-1}} \right\|_2 \leq 1 \right)
\end{align*}

Now we need to bound $\Exv{\|\phi_t\|^2}$. We notice that by Taylor's theorem, for some $z$ on the segment between $w_t$ and $w^*$,
\begin{align*}
  \norm{
    \nabla f(w_t)
    -
    H (w_t - w^*)
  }
  &=
  \norm{
    \nabla^2 f(z) (w_t - w^*)
    - 
    H (w_t - w^*)
  } \\
  &=
  \norm{
    \left( \nabla^2 f(z) - \nabla^2 f(w^*) \right) (w_t - w^*)
  } \\
  &\le
  \norm{ \nabla^2 f(z) - \nabla^2 f(w^*) } \norm{ w_t - w^* } \\
  &\le
  M \norm{ z - w^* } \norm{ w_t - w^* } \\
  &\le
  M \norm{ w_t - w^* }^2
\end{align*}
where the matrix norm used here is the induced 2-norm, and $M$ is our bound on the Lipschitz continuity of the second derivative of $f$.
It follows that
\begin{align*}
  \Exv{ \norm{
    \nabla f(w_t) 
    -
    H (w_t - w^*)
  }^2 }
  &\le
  M^2 \Exv{ \norm{ w_t - w^* }^4 }.
\end{align*}
From here, we can again apply our bound from Lemma~\ref{lemmaNewSingleIterBound}, which gives us
\begin{align*}
  \Exv{ \norm{
    \nabla f(w_t)
    -
    H (w_t - w^*)
  }^2 }
  &\le
  M^2 \cdot \frac{44^2 G^2 \delta^2 d}{\mu^2}.
\end{align*}

Putting back, we get 
\begin{align*}
    \Exv{\left\| \sum_{i=0}^{c-1} (I - \alpha H)^{c-i-1}\phi_{ct+i} \right\|^2}  
    \leq \left( \sum_{i=0}^{c-1} (1 - \alpha \mu )^{c-i-1} \right)^2 \frac{44^2M^2G^2\delta^2d}{\mu^2}
\end{align*}

Right now the bound for the second term becomes:
\begin{align*}
    & 3 \alpha^2\left(\frac{1}{K} \Exv{\left\|\sum_{i=0}^{c-1} (I - (I-\alpha H)^c)^{-1} (I - \alpha H)^{c-i-1} \phi_{ct+i}\right\|}\right)^2 \\
    & \quad \leq 
        3\alpha^2d
        \frac{44^2M^2G^2\delta^2d}{\mu^2}
        \left( \sum_{i=0}^{c-1} (1 - \alpha \mu )^{c-i-1} \right)^2 
        \|(I - (I - \alpha H)^c)^{-1}\|_2^2
\end{align*}
We know that $\|(I - (I - \alpha H)^c)^{-1}\|_2^2 \leq (1 - (1 - \alpha \mu)^c)^{-2}$, and
\[
\left( \sum_{i=0}^{c-1} (1 - \alpha \mu)^{c-i-1} \right)^2 
= \left( \sum_{i=0}^{c-1} (1 - \alpha \mu)^i \right)^2 
= \left( \frac{1-(1-\alpha\mu)^c}{1-(1-\alpha\mu)} \right)^2 
= \frac{(1- (1- \alpha\mu)^c)^2}{\alpha^2\mu^2}
\]
therefore, we have
\begin{align*}
    \left( \sum_{i=0}^{c-1} (1 - \alpha \mu )^{c-i-1} \right)^2 \|(I - (I - \alpha H)^c)^{-1}\|_2^2 
    \leq \frac{1}{\alpha^2\mu^2}
\end{align*}
With this, we can obtain the final bound on the second term substituting $\alpha = \sqrt{\frac{\delta^2 d}{G^2}}$
\[
    3 \alpha^2\left(\frac{1}{K} \Exv{\left\|\sum_{i=0}^{c-1} (I - (I-\alpha H)^c)^{-1} (I - \alpha H)^{c-i-1} \phi_{ct+i}\right\|}\right)^2 
    \leq 
    \frac{3\cdot 44^2 M^2 G^2 \delta^2 d}{\mu^4}
\]

Now we start analyzing the third term:
\begin{align*}
    & 3\Exv{\left\|\frac{1}{K}\sum_{t=0}^{K-1}\sum_{i=0}^{c-1}(I - (I - \alpha H)^c)^{-1} (I - \alpha H )^{c-i-1} \psi_{ct+i}\right\|^2} \\
    & \quad \leq 3\norm{ (I - (I - \alpha H)^c)^{-1} }_2^2 \Exv{\left\|\frac{1}{K}\sum_{t=0}^{K-1}\sum_{i=0}^{c-1} (I - \alpha H )^{c-i-1} \psi_{ct+i}\right\|^2} \\
    & \quad = 
    \frac{3}{K^2} \norm{ (I - (I - \alpha H)^c)^{-1} }_2^2 \sum_{t=0}^{K-1}\sum_{i=0}^{c-1} \Exv{\norm{ (I - \alpha H )^{c-i-1} \psi_{ct+i} }}
\end{align*}
The last equality is because all the cross term has expectation $0$ as $\Exv{\psi_t} = 0$ by using unbiased quantization $Q_\delta$. Following this line of analysis:
\begin{align*}
    & 3\Exv{\left\|\frac{1}{K}\sum_{t=0}^{K-1}\sum_{i=0}^{c-1}(I - (I - \alpha H)^c)^{-1} (I - \alpha H )^{c-i-1} \psi_{ct+i}\right\|^2} \\
    & \quad \leq 
    \frac{3}{K^2} \norm{ (I - (I - \alpha H)^c)^{-1} }_2^2 \sum_{t=0}^{K-1}\sum_{i=0}^{c-1} \norm{(I - \alpha H )^{c-i-1}}_2^2\Exv{\norm{  \psi_{ct+i} }} \\
    & \quad \leq 
    \frac{3}{K^2} (1-(1-\alpha\mu)^c)^{-2} \sum_{t=0}^{K-1}\sum_{i=0}^{c-1} ((1-\alpha\mu)^{c-i-1})^2 \Exv{\norm{  \psi_{ct+i} }}
\end{align*}
where the last line is followed by $\|(I - (I - \alpha H)^c)^{-1}\|_2^2 \leq (1-(1-\alpha\mu)^c)^{-2}$ and $\norm{ (I - \alpha H)^t }_2^2 \leq (1 - \alpha \mu)^{2t}$. 
Now we will bound $\Exv{\norm{ \psi_t }^2}$ as follow:
\begin{align*}
  \Exv{\norm{\psi_t}^2}
  &= \Exv{ \norm{
    Q_{\delta}\left(\alpha \nabla \tilde f_t(w_t) \right)
    -
    \alpha \nabla f(w_t)
  }^2 } \\
  &=
  \Exv{ \norm{
    Q_{\delta}\left(\alpha \nabla \tilde f_t(w_t) \right) 
    -
    \alpha \nabla \tilde f_t(w_t)
    +
    \alpha \nabla \tilde f_t(w_t)
    -
    \alpha \nabla f(w_t)
  }^2 } \\ 
  &=
  \Exv{ \norm{
    Q_{\delta}\left(\alpha \nabla \tilde f_t(w_t) \right)
    -
    \alpha \nabla \tilde f_t(w_t)
  }^2 }
  +
  \Exv{ \norm{
    \alpha \nabla \tilde f_t(w_t)
    -
    \alpha \nabla f(w_t)
  }^2 } \\
  &=
  \delta^2 d
  +
  \alpha^2 G^2 
  =
  2 \delta^2 d,
\end{align*}
where the last line applies $\alpha = \sqrt{\frac{\delta^2d}{G^2}}$. Putting this back, we have
\begin{align*}
    & 3\Exv{\left\|\frac{1}{K}\sum_{t=0}^{K-1}\sum_{i=0}^{c-1}(I - (I - \alpha H)^c)^{-1} (I - \alpha H )^{c-i-1} \psi_{ct+i}\right\|^2} \\
    & \quad \leq 
    \frac{3}{K^2} 2 \delta^2 d (1-(1-\alpha\mu)^c)^{-2} \sum_{t=0}^{K-1}\sum_{i=0}^{c-1} ((1-\alpha\mu)^{c-i-1})^2 \\
    & \quad \leq 
    \frac{6 \delta^2 d}{K}  (1-(1-\alpha\mu)^c)^{-2} \sum_{i=0}^{c-1} (1-\alpha\mu)^{2i} \\
    & \quad \leq 
    \frac{6 \delta^2 d}{K}  (1-(1-\alpha\mu)^c)^{-2} \left(\sum_{i=0}^{c-1} (1-\alpha\mu)^{i}\right)^2 
\end{align*}
Since $\alpha\mu < 1$ by Lemma~\ref{lemmaNoiseBall}. We know that $\sum_{i=0}^{c-1} (1-\alpha\mu)^i = \frac{1-(1-\alpha\mu)^c}{1-(1-\alpha\mu)}$, therefore the final bound for this term is:
\begin{align*}
    & 3\Exv{\left\|\frac{1}{K}\sum_{t=0}^{K-1}\sum_{i=0}^{c-1}(I - (I - \alpha H)^c)^{-1} (I - \alpha H )^{c-i-1} \psi_{ct+i}\right\|^2}  \leq \frac{6 \delta^2 d}{K\alpha^2\mu^2} = \frac{6G^2}{\mu^2K}
\end{align*}

Putting the three bounds together, we get
\begin{align*}
    \Exv{\|\bar{w} - w^*\|^2} \leq\frac{3\cdot 44^2 M^2 G^2 \delta^2 d}{\mu^4} + \frac{6G^2}{\mu^2K}+ \frac{528 G \delta \sqrt{d}}{\gamma \mu K^2} 
\end{align*}

Recall that $cK \leq T < cK+c$, we could replace $K$ with $T$ in the abovementioned bounds :
\begin{align*}
    \Exv{\|\bar{w} - w^*\|^2} \leq\frac{3\cdot 44^2 M^2 G^2 \delta^2 d}{\mu^4} + \frac{6G^2c}{\mu^2T}+ \frac{528 G \delta \sqrt{d}c^2}{\gamma \mu T^2} 
\end{align*}
which is what we want to prove.

\end{proof}

\section{Proof for Low-precision SGD Lower Bound}\label{sec:thm-lowerbound}
In previous work~\cite{training-quantized-network-deeper-understanding}, there were bounds on the size of the noise ball for low-precision SGD that looked like
\[
  \Exv{f(w_T) - f(w^*)} = O(\delta),
\]
where $\delta$ is the \emph{quantization gap} of the low-precision format chosen.
In our paper, SWALP, we found that doing weight averaging allows us to achieve something like
\[
  \Exv{f(w_T) - f(w^*)} = O(\delta^2).
\] 
which seemed to be a better convergence rate.
In order to show that it is actually better, we would need a lower bound on the noise ball size for low-precision SGD, not just an upper bound.
In this section, we will derive such a lower bound.

\begin{lemma}
  \label{lemmaCzbeta}
  There exists a constant $C > 0$ independent of problem parameters such that for all $z \in \R$ and for all $\beta > 0$,
  \[
    \mathbb{E}_{u \sim \mathcal{N}(0,1)}\left[\left( Q(z + \beta u) - (z + \beta u) \right)^2 \right]
    \ge
    C \cdot \min(1, \beta).
  \]
  Note that here, $Q$ with no subscript refers to random quantization onto the integers (i.e. $\delta = 1$).
\end{lemma}
\begin{proof}
  First, note that for any $x$, 
  \begin{align*}
    \Exv{ \left( Q(x) - x \right)^2 }
    &=
    \left( \lceil x \rceil - x \right)^2 \cdot \left( x - \lfloor x \rfloor \right)
    +
    \left( x - \lfloor x \rfloor \right)^2 \cdot \left( \lceil x \rceil - x \right) \\
    &=
    \left( \lceil x \rceil - x \right) \cdot \left( x - \lfloor x \rfloor \right) \left(
      \left( \lceil x \rceil - x \right)
      +
      \left( x - \lfloor x \rfloor \right)
    \right) \\
    &=
    \left( \lceil x \rceil - x \right) \cdot \left( x - \lfloor x \rfloor \right).
  \end{align*}
  So, we can re-write this as
  \[
  \mathbb{E}_{u \sim \mathcal{N}(0,1)}\left[ \left( Q(z + \beta u) - (z + \beta u) \right)^2 \right]
    =
    \mathbb{E}_{u \sim \mathcal{N}(0,1)}\left[ \left( \lceil z + \beta u \rceil - (z + \beta u) \right) \cdot \left( (z + \beta u) - \lfloor z + \beta u \rfloor \right) \right].
  \]
  Next, define the function
  \[
    \Phi(\beta, z)
    =
    \Exvud{u \sim \mathcal{N}(0,1)}{ \left( \lceil z + \beta u \rceil - (z + \beta u) \right) \cdot \left( (z + \beta u) - \lfloor z + \beta u \rfloor \right) }.
  \]
  Clearly, $\Phi$ is continuous.
  It is not difficult to see that
  \[
    \left( \lceil x \rceil - x \right) \cdot \left( x - \lfloor x \rfloor \right)
    \ge
    \frac{1}{4} \sin^2(\pi x)
    =
    \frac{1 - \cos(2 \pi x)}{8}.
  \]
  Therefore,
  \begin{align*}
    \Phi(\beta, z)
    &\ge 
    \Exvud{u \sim \mathcal{N}(0,1)}{ \frac{1 - \cos(2 \pi (z + \beta u))}{8} } \\
    &=
    \frac{1}{8} \Exvud{u \sim \mathcal{N}(0,1)}{ 1 - \cos(2 \pi z) \cos(2 \pi \beta u) + \sin(2 \pi z) \sin(2 \pi \beta u) } \\
    &=
    \frac{1}{8} \Exvud{u \sim \mathcal{N}(0,1)}{ 1 - \cos(2 \pi z) \cos(2 \pi \beta u) },
  \end{align*}
  where the last equality holds because $\sin$ is an odd function and the standard Normal distribution is even.
  Next, notice that
  \begin{align*}
    \Exvud{u \sim \mathcal{N}(0,1)}{ \cos(2 \pi \beta u) }
    &= 
    \Exvud{u \sim \mathcal{N}(0,1)}{ \frac{ \exp(i 2 \pi \beta u) +  \exp(-i 2 \pi \beta u) }{2} } \\
    &=
    \frac{ \phi_{\mathcal{N}}(2 \pi \beta) + \phi_{\mathcal{N}}(-2 \pi \beta)}{2},
  \end{align*}
  where $\phi_{\mathcal{N}}$ is the characteristic function of $\mathcal{N}(0,1)$, and is defined as
  \[
    \phi_{\mathcal{N}}(t) = \Exvud{u \sim \mathcal{N}(0,1)}{ \exp(i t u) }.
  \]
  The characteristic function of a standard Normal distribution is known to be
  \[
    \phi_{\mathcal{N}}(t) = \exp\left(-\frac{t^2}{2}\right),
  \]
  so
  \begin{align*}
    \Exvud{u \sim \mathcal{N}(0,1)}{ \cos(2 \pi \beta u) }
    &=
    \frac{ \exp(-2 \pi^2 \beta^2) + \exp(-2 \pi^2 \beta^2)}{2}
    =
    \exp(-2 \pi^2 \beta^2).
  \end{align*}
  Substituting this into our bound above,
  \begin{align*}
    \Phi(\beta, z)
    &\ge
    \frac{1}{8} \left( 1 - \cos(2 \pi z) \exp(-2 \pi^2 \beta^2) \right).
  \end{align*}
  This bound shows that $\Phi(\beta, z)$ is bounded away from zero everywhere except when $\beta = 0$ and $\cos(2 \pi z) = 1$.
  Since this expression is periodic in $z$, it suffices to consider just the case of $z = 0$.
  When $z = 0$, we have
  \[
    \Phi(\beta, 0)
    =
   \Exvud{u \sim \mathcal{N}(0,1)}{ \left( \lceil \beta u \rceil - \beta u \right) \cdot \left( \beta u - \lfloor \beta u \rfloor \right) }. 
  \]
  Expressing this as an integral, where $\psi$ is the probability density function of the standard normal distribution,
  \[
    \Phi(\beta, 0)
    =
    \int_{-\infty}^{\infty}
    \left( \lceil \beta u \rceil - \beta u \right) \cdot \left( \beta u - \lfloor \beta u \rfloor \right)
    \cdot \psi(u) \; du.
  \]
  Dividing both sides by $\beta$ and taking the limit as $\beta \rightarrow 0$,
  \begin{align*}
    \lim_{\beta \rightarrow 0}
    \frac{\Phi(\beta, 0)}{\beta}
    &=
    \lim_{\beta \rightarrow 0}
    \frac{1}{\beta}
    \int_{-\infty}^{\infty}
    \left( \lceil \beta u \rceil - \beta u \right) \cdot \left( \beta u - \lfloor \beta u \rfloor \right)
    \cdot \psi(u) \; du \\
    &=
    \int_{-\infty}^{\infty}
    \left(
      \lim_{\beta \rightarrow 0}
      \frac{1}{\beta} 
      \left( \lceil \beta u \rceil - \beta u \right) \cdot \left( \beta u - \lfloor \beta u \rfloor \right)
    \right)
    \cdot \psi(u) \; du \\
    &=
    \int_{-\infty}^{\infty}
    \Abs{u}
    \cdot \psi(u) \; du \\
    &=
    \sqrt{\frac{2}{\pi}}.
  \end{align*}
  Here, we can interchange the limit and integral by Lebesgue's Dominated Convergence Theorem.
  Now, since
  \[
    \frac{\Phi(\beta, z)}{\min(1, \beta)}
  \]
  is a continuous function, it follows from the extreme value theorem that it must attain its minimum value somewhere in its domain (reasoning about $z$ as living in a compact space since it is periodic, and reasoning about $\beta$ as living in a compact space with a point at infinity).
  But we've shown that at every point in its domain, this function is positive.
  So it follows that its minimum value must also be some number greater than zero.
  Call this number $C$.
  Then, it follows that
  \[
    \Exvud{u \sim \mathcal{N}(0,1)}{ \left( Q(z + \beta u) - (z + \beta u) \right)^2 }
    \ge
    C \cdot \min(1, \beta)
  \]
  which is what we wanted to prove.
\end{proof}

\begin{customthm}{\ref{thm:lower-bound}}\label{thm:lower-bound-appendix}
Consider one-dimensional objective function $f(x) = \frac{1}{2}x^2$ with gradient samples $\tilde{f}'(w) = w + \sigma u$ where $u \sim \mathcal{N}(0,1)$. Compute $w_T$ recursively using the quantized SGD updating step : $w_{t+1} = Q_\delta(w_t - \alpha \tilde{f}'(w_t))$. There exists a constant $A > 0$ such that for all step size $\alpha > 0$, 
$\lim_{T\to \infty} \mathbb{E}[w_T^2] \geq \sigma \delta A$.
\end{customthm}

\begin{proof}
  Let $\tilde{f}'(w_t) = w_t + \sigma u_t$ where $u_t\sim \mathcal{N}(0,1)$.
  Computing the expected value of this squared at the next time-step,
\begin{align*}
  \Exv{w_{t+1}^2}
  &=
  \Exv{ \left( Q_{\delta}\left((1 - \alpha) w_t + \alpha \sigma \tilde u_t \right) \right)^2} \\
  &=
  \Exv{ \left( (1 - \alpha) w_t + \alpha \sigma \tilde u_t \right)^2}
  +
  \Exv{ \left( Q_{\delta}\left((1 - \alpha) w_t + \alpha \sigma \tilde u_t \right) - \left((1 - \alpha) w_t + \alpha \sigma \tilde u_t \right)\right)^2} \\
  &=
  (1 - \alpha)^2 \Exv{ w_t^2 } + \alpha^2 \sigma^2
  +
  \delta^2
  \Exv{ \left( Q\left(\frac{(1 - \alpha) w_t + \alpha \sigma \tilde u_t}{\delta} \right) - \left(\frac{(1 - \alpha) w_t + \alpha \sigma \tilde u_t}{\delta} \right)\right)^2},
\end{align*}
where our reduction in the last line follows because $\tilde u_t$ is a standard normal random variable and so $\Exv{\tilde u_t^2} = 1$.
Applying Lemma~\ref{lemmaCzbeta} to bound the last term, we get
\begin{align*}
  \Exv{w_{t+1}^2}
  &\ge
  (1 - \alpha)^2 \cdot \Exv{ w_t^2 } + \alpha^2 \sigma^2
  +
  C \delta^2 \cdot \min\left(1, \frac{\alpha \sigma}{\delta} \right),
\end{align*}
Let $W$ denote the fixed point of this expression.
Then
\[
  W
  =
  (1 - \alpha)^2 \cdot W + \alpha^2 \sigma^2
  +
  C \delta^2 \cdot \min\left(1, \frac{\alpha \sigma}{\delta} \right),
\]
and so
\begin{align*}
  W
  &=
  \frac{\alpha^2 \sigma^2}{2\alpha - \alpha^2}
  +
  \frac{C \delta^2}{2 \alpha - \alpha^2} \cdot \min\left(1, \frac{\alpha \sigma}{\delta} \right) \\
  &=
  \frac{\alpha \sigma^2}{2 - \alpha}
  +
  \min\left(\frac{C \delta^2}{2 \alpha - \alpha^2}, \frac{\sigma C \delta}{2 - \alpha} \right).
\end{align*}
If $0 < \alpha < 2$ (that is, $\alpha$ is actually small enough that the SGD is stable), then we can bound this with
\begin{align*}
  W
  &\ge
  \frac{\alpha \sigma^2}{2}
  +
  \min\left(\frac{C \delta^2}{2 \alpha}, \frac{\sigma C \delta}{2} \right) \\
  &=
  \min\left(
    \frac{\alpha \sigma^2}{2} + \frac{C \delta^2}{2 \alpha}
  , 
    \frac{\alpha \sigma^2}{2} + \frac{\sigma C \delta}{2}
  \right) \\
  &\ge
  \min\left(
    \sigma \delta \sqrt{C}
  , 
    \frac{\sigma \delta C}{2}
  \right),
\end{align*}
where this last expression comes from minimizing each of the terms individually with respect to $\alpha$, subject to $\alpha > 0$.
If we define the problem-independent constant
\[
  A = \min\left(
    \sqrt{C}
  , 
    \frac{C}{2}
  \right),
\]
then we get
\[
  W \ge \sigma \delta A.
\]
Returning to our bound on the squares of the iterates, and subtracting the fixed point from both sides,
\begin{align*}
  \Exv{w_{t+1}^2 - \sigma \delta A}
  &\ge
  (1 - \alpha)^2 \cdot \Exv{ w_t^2 - \sigma \delta A}.
\end{align*}
Applying this inductively, and taking the limit, we get
\[
  \lim_{T \rightarrow \infty}
  \Exv{w_T^2}
  \ge
  \sigma \delta A
  =
  O(\delta).
\]

\end{proof}
This is a lower bound that proves we can't in general do better than $O(\delta)$ for low-precision SGD, even for strongly convex models.

\newpage

\section{Combining SWALP with low-precision training method}\label{sec:wage-swalp}
In Section~\ref{sec:related}, we introduce SWALP as an orthogonal approach to recent low-precision training algorithms. In this section, we explore the possibility of combining SWALP with a state-of-the-art low-precision model, WAGE~\cite{WAGE}. 
We evaluate the performance of SWALP and that of the modified low-precision SGD by training the WAGE network on CIFAR10 \cite{WAGE}. 
Originally, WAGE is trained with learning rate $8$ for the first 200 epochs and decay to $1$ at the 200 epochs, and to $0.125$ at the 250 epochs. Based on a held-out validation set, we find the learning rate schedule for SWALP. 
Specifically, we use a constant learning rate $8$ for the first 200 epochs and start SWALP at 250 epochs with a constant SWA learning rate $6$ and averaging frequency $c=1$. 
We keep other hyper-parameters and training procedures unchanged.
During inference, we use the full precision average as the weight in forward propagation and keep activations in 8 bits by following WAGE's quantization scheme.
We denote the WAGE network trained with SWALP, WAGE-SWALP, and report its test error rate on CIFAR10 in Table~\ref{tab:wage-swalp}. 
Due to the customized scaling in WAGE, it is nontrivial to train WAGE with full precision and we therefore omit the comparison with full precision baseline. 
Table~\ref{tab:wage-swalp} shows that SWALP can be combined with the existing low-precision training algorithm positively. We also note that SWALP requires little modification on the network structure and hyper-parameters except a different learning rate schedule.

\begin{table}[H]
    \centering
    \captionsetup{justification=centering}
    \caption{Test Error (\%) on CIFAR10. SWALP has positive interaction with state-of-the-art low-precision training algorithm.}
    \begin{tabular}{cc}
        \toprule
        Model & Test Error  \\
        \midrule
        WAGE~\cite{WAGE} & 6.78 \\
        WAGE-SWALP & 6.35 $\pm{0.04}$ \\
        \bottomrule
    \end{tabular}\label{tab:wage-swalp}
\end{table}

\section{Implementation Details for Linear Regression}\label{sec:linreg}

\textbf{Objective Function.} 
In this experiment, we define the objective function as $f(w) = \frac{1}{n}\sum_{i=1}^n (w^Tx_i - y_i)^2$ and $\tilde{f}(w) = (w^Tx_i - y_i)^2$ for a randomly sampled $i$.

\textbf{Synthetic Dataset.}
The data points $x_i$ were generated by $x_i \sim \mathcal{N}(0, \sigma_x^2I)$. 
We then chose target weights $w_{\text{init}}$ uniformly at random in the range $[-1, 1]$, and generated the labels by $y_i \sim \mathcal{N}(w_{init}^T x_i, \sigma_u^2)$.
To generate the synthetic dataset, we set $d = 256$ and $\sigma_u = \sigma_x = 1$ and sampled $4096$ data points.

\begin{figure*}[t]
    \centering
    \subfigure[Log-log plot of linear regression]{
        \centering
        \includegraphics[width=0.4\textwidth]{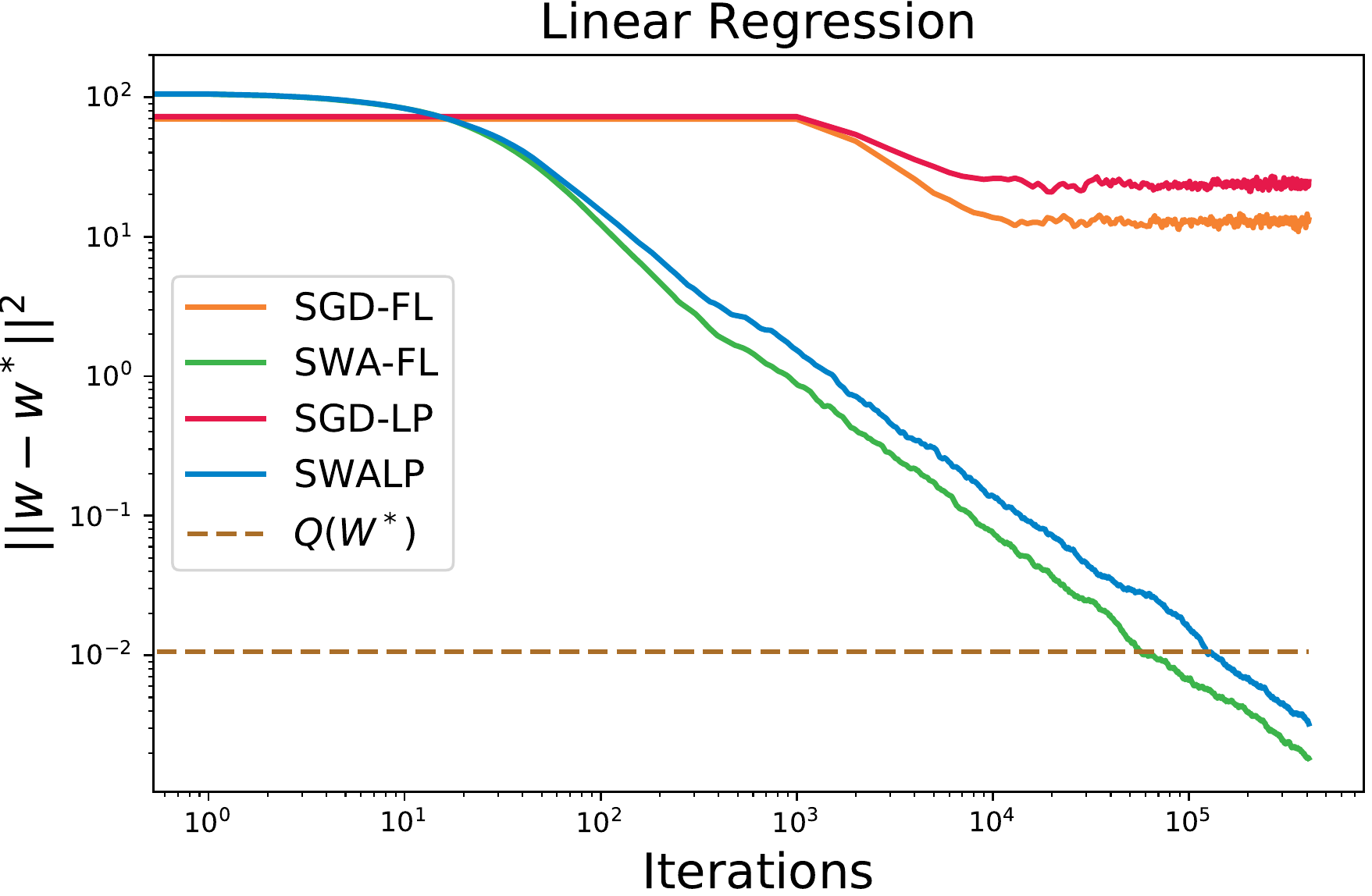}\label{fig:log-log}
    }
    \subfigure[MNIST test error (\%) for logistic regression experiments. We varied the precisions used to train SGD-LP and SWALP.]{
        \centering
        \includegraphics[width=0.45\textwidth]{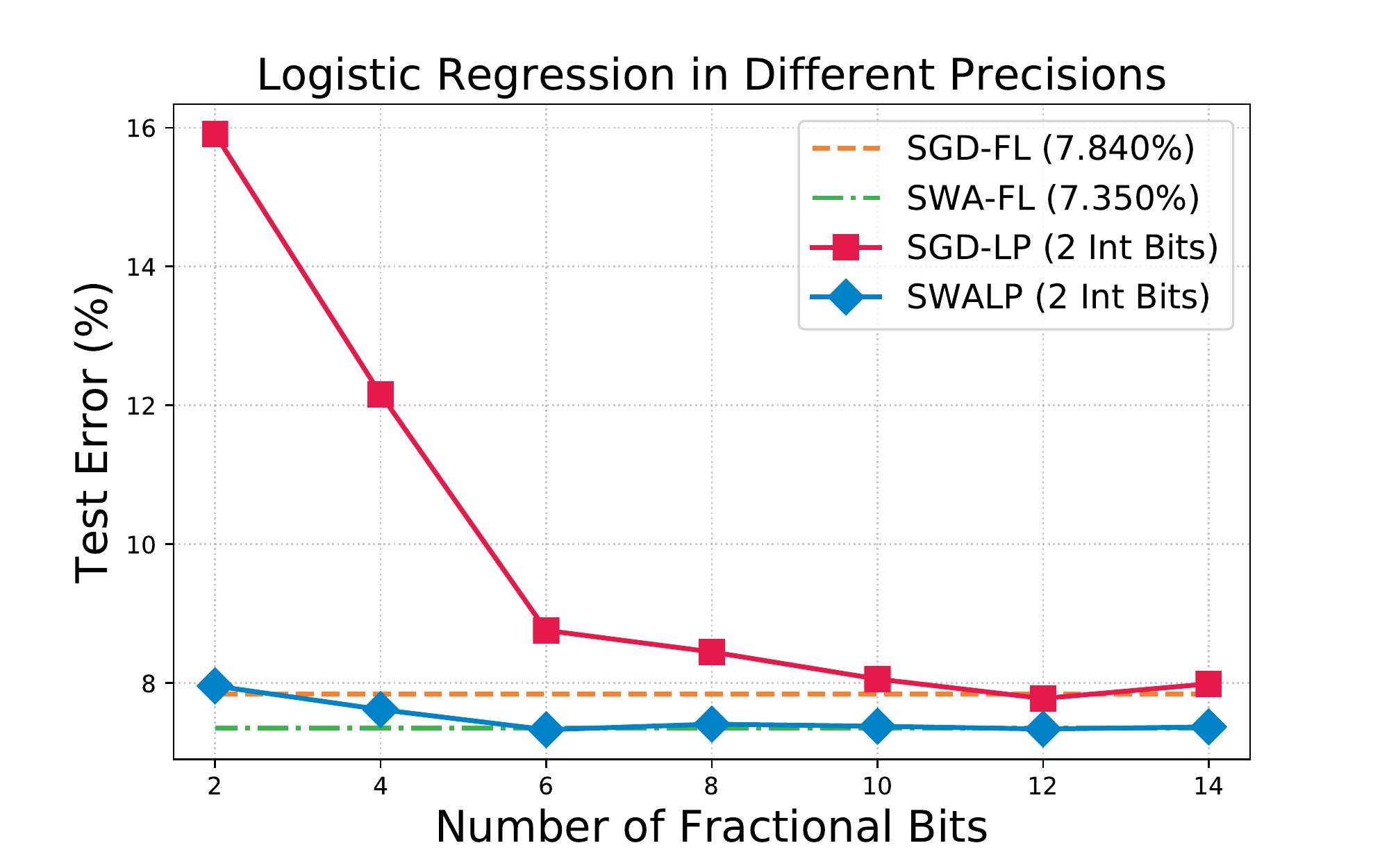}\label{fig:logreg-diff-prec-test}
    }
    \caption{More experimental results for linear and logistic regression.}
\end{figure*}

\textbf{Convergence.}
We plot the square distance between $w_t$ (or $\bar w_t$ for SWALP) and the optimal solution $w^*$ in log-log scale in Figure~\ref{fig:log-log}.
For reference, we also plot the squared distance between $Q(w^*)$ and $w^*$ to illustrate the size of quantization noise. It can be seen more clearly that the convergence rate is about $O(1/T)$.

\section{Implementation Details for Logistic Regression}\label{sec:logreg}

We use logistic regression with L2 regularization on MNIST dataset ~\cite{MNIST} .
In this experiment, we use the MNIST dataset~\cite{MNIST}. 
The objective function is defined as
$
    f(w) = - \frac{1}{n}\sum_{i=1}^n \log{(\operatorname{softmax}(w^Tx_i + b))} + \frac{\lambda}{2} \|w\|^2
$.
We choose $\lambda=10^{-4}$, a regularization parameter used in prior work~\cite{HALP,SVRG}, which makes the objective a strongly convex function with $M\ne 0$.
We use a learning rate of $\alpha = 0.01$ and cycle length $C=1$ for all four settings.
For this experiment we measure the norm of gradient at each iteration to illustrate the convergence of the algorithm; this is a more useful metric because MNIST is sparse and poorly conditioned, and it is a metric that has been used for logistic regression on MNIST in previous work~\cite{HALP}.
We warm up for 60000 iterations. 
For SWALP and SGD-LP, we use 4-bit word length and 2-bit fractional length.
For the experiment where we evaluate different precision, we use the same hyper-parameters, except the following: 1) the warm-up period is set to 600k iterations (i.e. 10 epochs); and 2) we report the final evaluation results after training for 3 million steps (i.e. 50 epochs). The test results is reported in Figure~\ref{fig:logreg-diff-prec-test}. As we could see the same conclusion still holds. Detail statistics is reported in Table~\ref{table:logreg-diffprec-stats}.

\begin{table}[H]
\centering
\caption{MNIST training and testing error (\%) for logistic regression experiment with different fractional bits for training.} 
\begin{tabular}{lccccc}

\toprule
 & & \multicolumn{2}{c}{SGD} & \multicolumn{2}{c}{SWA} \\
Format & Precision & Train Err   & Test Err  & Train Err   & Test Err  \\
\midrule
Float & 32 & 7.07 & 7.84 & 6.6 & 7.35 \\
\multirow{7}{*}{Fixed Point} & FL=14, WL=16 & 7.30 & 7.99 & 6.59 & 7.37 \\
 & FL=12, WL=14 & 7.18 & 7.78 & 6.59 & 7.34 \\
 & FL=10, WL=12 & 7.21 & 8.06 & 6.57 & 7.38 \\
 & FL=8, WL=10 & 7.82 & 8.45 & 6.61 & 7.41 \\
 & FL=6, WL=8 & 8.31 & 8.76 & 6.78 & 7.33 \\
 & FL=4, WL=6 & 11.64 & 12.16 & 7.21 & 7.62 \\
 & FL=2, WL=4 & 16.27 & 15.91 & 7.98 & 7.96 \\
\bottomrule
\end{tabular}
\label{table:logreg-diffprec-stats}
\end{table}

\section{Implementation Details for Section~\ref{sec:expr}}\label{sec:dlexp}
For VGG-16 and Pre-activation ResNet-164, we replicate the full precision baselines by using the publicly released implementation of SWA~\cite{SWA-repo}. Based on validation error, we discover that for low-precision training it is beneficial to decay the learning rate lower before SWA starts. Therefore, we follow the SGD learning rate schedule described in \citet{SWA} before SWALP starts. For ResNet-18 on ImageNet dataset, we follow the learning rate schedule described in \citet{resnet}.

We now disclose the specific hyperparameters.
\begin{itemize}
    \item For VGG-16 reported in \citet{SWA}, we use He's initialization \cite{he-init} and weight decay of \texttt{5e-4}. One budget is set to be $200$ epochs. For SGD-LP, we set the initial learning rate $\alpha_1=0.05$ for the first $0.5$ budget, and linearly decreased to $0.01 \alpha_1$ for $0.5-0.9$ budget, and keep the learning rate constant for $0.9-1.0$ budget. For SWALP, we use the same learning rate for the first $200$ epochs. We starts SWALP averaging at $200$th epoch and keep a constant learning rate $0.01$ with averaging frequency $c=1$ epoch. 
    \item For Preactivation ResNet-164, we use He's initialization \cite{he-init} and weight decay of \texttt{3e-4}. One budget is set to be $150$ epochs. For SGD, we use cycle length $C=1$ and set the initial learning rate $\alpha_1=0.1$ for the first $0.5$ budget, and linearly decreased to $0.01 \alpha_1$ for $0.5-0.9$ budget, and keep the learning rate constant for $0.9-1.0$ budget. For SWALP, we start averaging at $150$ epoch and keep a constant learning rate $0.01$ with averaging frequency $c=1$ epoch. Moreover, we found that possibly because low-precision training is less stable, it is better to average for less epochs. Therefore, we evaluate the full precision averaged model on a subset of the training set and report the performance of the model with lowest loss on this subset. 
    \item For ResNet-18 on ImageNet, we use weight decay \texttt{1e-4}. One budget is set to be 90 epochs. For both SGD and SGD-LP, we use initial learning rate $0.1$ and decay it by a factor of 10 every 30 epochs. For SWA and SWALP, we start averaging after 90 epochs with a constant learning rate of $0.001$ for SWA and $0.001$ for SWALP (obtained by a grid search). The averaging frequency is $c=1$ epochs. 
\end{itemize}
    
\clearpage
\section{Data for Figure~\ref{fig:diffc_qswa}}\label{sec:data-figure3}

\begin{table*}[h]
\label{table:figure3-data-diffc}

\caption{CIFAR100 classification error (\%) on test set. We use the same base model and average it with different frequency. Each column show how the SWA model perform on the test set during the training according to each averaging schedule. }
\centering
\begin{tabular}{@{}lllllllllllllll@{}}
\toprule
Epoch   & \multicolumn{2}{c}{1} & \multicolumn{2}{c}{5} & \multicolumn{2}{c}{10} & \multicolumn{2}{c}{50} & \multicolumn{2}{c}{80} & \multicolumn{2}{c}{90} & \multicolumn{2}{c}{100} \\
\#batch/avg & Error         & STD        & Error         & STD        & Error          & STD        & Error          & STD        & Error          & STD        & Error          & STD        & Error          & STD         \\ \midrule
1                 & 28.95       & 0.12       & 28.38       & 0.12       & 28.01        & 0.10       & 27.20        & 0.34       & 27.21        & 0.13       & 27.19        & 0.13       & 27.11        & 0.29        \\
2                 & 28.96       & 0.12       & 28.37       & 0.13       & 28.01        & 0.10       & 27.20        & 0.34       & 27.21        & 0.13       & 27.19        & 0.13       & 27.12        & 0.29        \\
20                & 28.96       & 0.24       & 28.32       & 0.12       & 28.02        & 0.04       & 27.20        & 0.33       & 27.21        & 0.11       & 27.20        & 0.15       & 27.14        & 0.28        \\
100               & 29.20       & 0.21       & 28.47       & 0.09       & 28.23        & 0.18       & 27.24        & 0.29       & 27.18        & 0.23       & 27.17        & 0.17       & 27.18        & 0.22        \\
200               & 29.63       & 0.31       & 28.74       & 0.15       & 28.25        & 0.14       & 27.27        & 0.29       & 27.23        & 0.21       & 27.17        & 0.14       & 27.20        & 0.20        \\
once/epoch                & 30.00       & 0.12       & 28.88       & 0.13       & 28.28        & 0.10       & 27.23        & 0.31       & 27.14        & 0.16       & 27.19        & 0.15       & 27.16        & 0.26        \\ \bottomrule
\end{tabular}
\end{table*}

\begin{table*}[h]
\label{table:figure3-data-qswa}
\caption{CIFAR100 classification error (\%) on test set. We use block floating point to quantize the average. Each column shows the result of using different number bits for the block floating point. }
\centering
\begin{tabular}{cccccccccc}
\toprule
\# of bits & Float            & 16               & 14               & 12               & 10               \\
Test Error & 26.65$\pm{0.29}$ & 26.65$\pm{0.29}$ & 26.66$\pm{0.29}$ & 26.65$\pm{0.29}$ & 26.60$\pm{0.33}$ \\
\midrule
\# of bits &  & 9                & 8                & 7                & 6 \\
Test Error &  & 26.67$\pm{0.28}$ & 26.85$\pm{0.38}$ & 27.48$\pm{0.38}$ & 29.48$\pm{0.25}$  \\
\bottomrule
\end{tabular}
\end{table*}

\end{appendices}

\end{document}